\documentclass{article}
\usepackage{kr}

\usepackage{times}
\usepackage{soul}
\usepackage{url}
\usepackage[hidelinks]{hyperref}
\usepackage[utf8]{inputenc}
\usepackage[small]{caption}
\usepackage{graphicx}
\usepackage{amsmath} 
\usepackage{amsthm}
\usepackage{booktabs}
\usepackage{algorithm} 
\usepackage{cutwin}
\usepackage{stmaryrd}   
\usepackage{amssymb}
\usepackage{bm}
\usepackage{paralist}
\usepackage[inline]{enumitem}
\usepackage[all]{foreign}
\usepackage{tikz}
\usepackage{forest}
\usepackage{subfig}
\usepackage{algpseudocode}
\usepackage{pgfplots}
\usepackage{pgfplotstable}
\pgfplotsset{compat=1.17}

\algnotext{EndIf}
\algnotext{EndFor}
\algnotext{EndWhile}
\algnotext{EndFunction}
\algrenewcommand\algorithmicindent{9pt}

\usetikzlibrary{arrows,positioning,calc}
\pgfplotstableset{col sep=semicolon,text special chars={\#}}

\definecolor{palette1}{HTML}{FA686A}
\definecolor{palette2}{HTML}{FF8352}
\definecolor{palette3}{HTML}{FECB02}
\definecolor{palette4}{HTML}{57C78B}
\definecolor{palette5}{HTML}{71D6CA}
\definecolor{palette6}{HTML}{4D9DE0}
\definecolor{palette7}{HTML}{7776BC}
\definecolor{palette8}{HTML}{FEC1C9}
\definecolor{palette9}{HTML}{C800FF}

\usepackage{todonotes}
\usepackage{ifthen}

\newboolean{expandedVersion}
\setboolean{expandedVersion}{true}

\usepackage{trimclip}

\newcommand{\B}[1]{\Box_{#1}}
\newcommand{\D}[1]{\Diamond_{#1}}

\def\UpdateNode{ChildNode}

\newcommand{\atom}[1]{#1}

\newcommand{\atomSet}{\mathcal{P}}

\newcommand{\agentSet}{\mathcal{AG}}
\newcommand{\actionSet}{\mathcal{A}}

\newcommand{\E}{A} 
\newcommand{\Lang}{\mathcal{L}}

\newcommand{\md}{\mathit{md}}

\newcommand{\depth}[1]{d(#1)}
\newcommand{\bound}[1]{b(#1)}

\newcommand{\bisim}{\leftrightarroweq} 

\newcommand{\contr}[2]{{\lfloor #1 \rfloor}_{#2}}
\newcommand{\rootedContr}[3]{{\llfloor #1 \rrfloor}_{#2}^{#3}}
\newcommand{\canonicalContr}[2]{\rootedContr{#1}{#2}{\star}}

\newcommand{\class}[2]{[#1]_{#2}}
\newcommand{\reprClass}[1]{\llbracket #1 \rrbracket}

\newcommand{\maxRepr}[1]{#1^{\max}}
\newcommand{\strictMaxRepr}[1]{\maxRepr{#1}_{>0}}
\newcommand{\leastRepr}[2]{\min^{<}_{#2}(#1)}

\newcommand{\Sign}[2]{\sigma_{#2}(#1)}
\newcommand{\SignMap}[3]{\Sigma_{#3}(#1, #2)}

\newcommand{\reprSign}[1]{\Sign{#1}{}}
\newcommand{\canonicalSign}[2]{\sigma^\star_{#2}({#1})}

\def\IBDS{\textsc{ibds}}   

\newcommand{\graphPath}[1]{\mathit{path}(#1)}

\tikzstyle{world} =[circle,   thick,draw=black,       fill=black,minimum size=5pt,inner sep=0pt]
\tikzstyle{dworld}=[circle,   thick,draw=black,double,fill=black,minimum size=5pt,inner sep=0pt]

\tikzstyle{event} =[rectangle, thick,draw=black,       fill=black,minimum size=5pt,inner sep=0pt]
\tikzstyle{devent}=[rectangle, thick,draw=black,double,fill=black,minimum size=5pt,inner sep=0pt]

\newcommand{\myparagraph}[1]{\medskip\noindent\textbf{#1}.}

\newtheorem{example}{Example}
\newtheorem{theorem}{Theorem}
\newtheorem{definition}{Definition} 
\newtheorem{lemma}{Lemma}
\newtheorem{proposition}{Proposition}

\title{Depth-Bounded Epistemic Planning}

\author{%
    Thomas Bolander$^1$\and
    Alessandro Burigana$^2$\and
    Marco Montali$^2$\\
    \affiliations
    $^1$Technical University of Denmark, Denmark\\
    $^2$Free University of Bozen-Bolzano, Italy\\
    \emails
    tobo@dtu.dk,
    alessandro.burigana@unibz.it,
    montali@inf.unibz.it
}

\begin{document}
    \maketitle

    \begin{abstract}
    We propose a novel algorithm for epistemic planning based on dynamic epistemic logic (DEL).
    The novelty is that we limit the depth of reasoning of the planning agent to an upper bound $b$, meaning that the planning agent can only reason about higher-order knowledge to at most (modal) depth $b$.
    We then compute a plan requiring the lowest reasoning depth by iteratively incrementing the value of $b$.
    The algorithm relies at its core on a new type of ``canonical'' $b$-bisimulation contraction that guarantees unique minimal models by construction.
    This yields smaller states wrt.\ standard bisimulation contractions, and enables to efficiently check for visited states.
    We show soundness and completeness of our planning algorithm, under suitable bounds on reasoning depth, and that, for a bound $b$, it runs in $(b{+}1)$-\textsc{ExpTime}.
    We implement the algorithm in a novel epistemic planner, \textsc{daedalus}, and compare it to the \textsc{efp 2.0} planner on several benchmarks from the literature, showing effective performance improvements.
\end{abstract}

    \section{Introduction}
    \emph{Automated planning} is central in AI research 
    \cite{books/elsevier/Gallab2004,series/synthesis/2013Geffner}. 
    In fully observable and deterministic domains,  
    the output of a planning algorithm (a \emph{planner}) is a sequence of actions achieving the desired goal state from a given initial state.  
     \emph{Epistemic planning} is the enrichment of planning with epistemic notions, in particular knowledge and belief, including \emph{higher-order knowledge},
    \ie knowledge about what other agents know, knowledge about what someone knows about someone else, etc.  
    One of the reference frameworks for epistemic planning is dynamic epistemic logic (DEL) \cite{conf/tark/Batlag1998} in which epistemic states are represented as Kripke models. \emph{DEL planning}    
    was first pursued by \citeauthor{journals/jancl/Bolander2011}~\shortcite{journals/jancl/Bolander2011}, but independently conceived by others~\cite{conf/lori/Lowe2011,pardo2012planning,aucher2012del}. Recently, a special issue of the journal Artificial Intelligence (AIJ) was devoted to epistemic planning~\cite{belle2022epistemic}. 
    
    One of the main challenges in epistemic planning is its high computational complexity: unrestricted DEL planning has an undecidable plan existence problem~\cite{journals/jancl/Bolander2011}. This has led researchers to seek more tractable fragments, both theoretically and practically. Intuitively, a main source of complexity is the fact that DEL planning allows agents to reason about others' (higher-order) knowledge up to any nesting level. Thus, limiting reasoning depth seems a promising approach to tame DEL planning's complexity.
    This led to decidable fragments by, e.g., restricting actions to be propositional, so that the required reasoning depth is bound by the modal depth of the goal formula~\cite{conf/ijcai/Yu2013}, or by imposing axioms on the semantics leading to collapse higher-order reasoning to lower-order~\cite{buriganaAPM/2023/Semantic}.
    An alternative to DEL planning is the \emph{sentential approach}, where states are modelled as knowledge (or belief) bases, \ie sets of formulas. Here, decidable fragments have been single out by explicitly bounding the modal depth of formulas in states~\cite{conf/aaai/Muise2015,muiseBFMMPS/2022/Efficient}.
 
   Our paper is also concerned with taming the computational complexity of epistemic planning via restricting the reasoning bound. However,  
   it is not about achieving a well-behaved, e.g.\ decidable, fragment of epistemic planning by considering a framework with fixed reasoning bounds. Rather, we propose an alternative algorithmic approach that maintains the generality of unrestricted DEL planning, but always computes a plan using the lowest
   reasoning depth. This is achieved via an ``iterative deepening'' algorithm that iteratively increments the allowed reasoning depth until a solution is found. Thus we preserve full generality while still achieving the computational benefits of keeping the reasoning depth as low as possible, hence giving us an algorithm that ``degrades gracefully'' with the increased reasoning bounds required by epistemically intricate planning tasks. For instance, an agent might activate complex higher-order reasoning when playing an epistemic card game like Hanabi, but not when making coffee. If making coffee does not require any higher-order reasoning at all (\ie, it reasoning about the mental states of other agents is not needed), our algorithm will find a plan at reasoning depth 0, collapsing all the epistemic states into propositional ones.  
  

    Our contribution is fourfold: 
    \begin{inparaenum}
        \item We introduce canonical $b$-bisimulation contractions as a compact, depth-bounded representation of states;
        \item We define an iterative-deepening algorithm for computing plans with the lowest possible modal depth;
        \item We show soundness, completeness and complexity results; and
        \item We implement the algorithm in our novel epistemic planner, \textsc{daedalus}, and show effective improvements wrt.\ the \textsc{efp 2.0} planner~\cite{conf/icaps/Fabiano2020}.
    \end{inparaenum}




\section{Preliminaries}\label{sec:del}
    \paragraph{Dynamic Epistemic Logic}
    We recall the main concepts of DEL. For a more complete account, see \citeauthor{book/springer/vanDitmarsch2007}~\shortcite{book/springer/vanDitmarsch2007}.
        Let $\atomSet$ be a finite set of \emph{atomic propositions} (\emph{atoms}) and $\agentSet$
        a finite set of \emph{agents}. The language $\Lang$ of \emph{multi-agent epistemic logic} is given by: 
            $\phi ::= \atom{p} \mid \neg \phi \mid \phi \wedge \phi \mid \B{i} \phi$
        ($\atom{p} \in \atomSet$, $i \in \agentSet$).
        Formula $\B{i} \phi$ is read as ``agent $i$ knows/believes that $\phi$''.
        Symbols $\vee$, $\rightarrow$ and $\D{i}$ are defined by abbreviation as usual. 

        \begin{definition}[States]
            An \emph{(epistemic) model} of $\Lang$ is a triple $M = (W, R, L)$ where:
                1) $W \neq \varnothing$ is a finite set of \emph{(possible) worlds}; 
                2) $R: \agentSet \to 2^{W \times W}$ maps to each agent $i$ an \emph{accessibility relation} $R_i$;
                3)
                $L: W \rightarrow 2^\atomSet$ maps to each world $w$ a \emph{label} $L(w)$. 
            An \emph{(epistemic) state} of $\Lang$ is a pair $s = (M, w)$, where $w \in W$ is the \emph{actual world}.
        \end{definition}

        \noindent 
        We often write $w R_i v$ for $(w, v) \in R_i$.
        Truth of formulas $\phi \in \Lang$ in state $(M,w)$ is defined inductively as follows, 
        with the standard clauses for propositional connectives: 

        {\centering
            $\begin{array}{@{}lll}
                   (M, w) \models \atom{p}            & \text{ iff } & p \in L(w) \\           
                   (M, w) \models \B{i} \phi       & \text{ iff } & \text{for all } v, \text{ if } w R_i v \text{ then } (M, v) \models \phi
               \end{array}$
        \par}
        \begin{figure}
            \centering
            \scalebox{0.9}{
\begin{tikzpicture}[>=latex,node distance=4.2em,style={font=\sffamily\footnotesize},align=center]
    \node[world,              label={below:{$w_0 {:}$ \\ $\mathit{has}(a, 5)$ \\
    $\mathit{has}(b, 4)$
    }}] (w0) {};
    \node[dworld, right=of w0, label={below:{$w_1 {:}$ \\ $\mathit{has}(a, 3)$ \\ 
    $\mathit{has}(b, 4)$
    }}] (w1) {};
    \node[world , right=of w1, label={below:{$w_2 {:}$ \\ $\mathit{has}(a, 3)$ \\
    $\mathit{has}(b, 2)$
    }}] (w2) {};
    \node[world , right=of w2, label={below:{$w_3 {:}$ \\ 
    $\mathit{has}(a, 1)$ \\
    $\mathit{has}(b, 2)$}}] (w3) {};
    \node[world , right=of w3, label={below:{$w_4 {:}$ \\ $\mathit{has}(a, 1)$
    \\ $\mathit{has}(b, 0)$
    }}] (w4) {};
    \node[left=2em of w0,xshift=2mm] {$s_0 = $};

    \begin{scope}[inner sep=2pt]
        \path
            (w0) edge[<->] node[inner sep=0, label={above:{$b$}}] {} (w1)
            (w1) edge[<->] node[inner sep=0, label={above:{$a$}}] {} (w2)
            (w2) edge[<->] node[inner sep=0, label={above:{$b$}}] {} (w3)
            (w3) edge[<->] node[inner sep=0, label={above:{$a$}}] {} (w4)
            (w0) edge[->, loop, out=55, in=125, looseness=12] node[inner sep=0, label={above:{$a, b$}}] {} (w0)
            (w1) edge[->, loop, out=55, in=125, looseness=12] node[inner sep=0, label={above:{$a, b$}}] {} (w1)
            (w2) edge[->, loop, out=55, in=125, looseness=12] node[inner sep=0, label={above:{$a, b$}}] {} (w2)
            (w3) edge[->, loop, out=55, in=125, looseness=12] node[inner sep=0, label={above:{$a, b$}}] {} (w3)
            (w4) edge[->, loop, out=55, in=125, looseness=12] node[inner sep=0, label={above:{$a, b$}}] {} (w4)
        ;
    \end{scope}
\end{tikzpicture}
}
            \caption{Epistemic state $s_0$ of Example \ref{ex:state}. Bullets represent worlds, labelled by their name and the atoms they satisfy. The actual world is circled. Edges represent the accessibility relations. 
            }
            \label{fig:state}
        \end{figure}
        \begin{example}[Consecutive Number Puzzle \cite{book/sip/vanDitmarschK2015}]
            \label{ex:state}
            Agents $a$ and $b$ are given two \emph{consecutive} numbers $n_a, n_b \in [0,N]$, only seeing their own number.
            Let $\mathit{has}(i, n)$ denote that agent $i$ has number $n$.
            State $s_0$ (Figure~\ref{fig:state}) has $N = 5$, $n_a = 3$ and $n_b = 4$. 
            We have $s_0 \models \B b \bigl( \mathit{has}(b,4) \land (\mathit{has}(a,3 ) \lor  \mathit{has}(a,5))  \bigr)\land \neg \B b \mathit{has}(a,3 ) \land \neg \B b \mathit{has}(a,5)$: $b$ knows herself to have $4$ and knows $a$ to  have either $3$ or $5$, but $b$ doesn't know which of the two. 
        \end{example}

        In DEL, actions are represented by \emph{event models}. These are Kripke structures on a set of \emph{events}, where each event represents a possible perspective on an action defined via a \emph{precondition} and \emph{postconditions}.   
        \begin{definition}[Actions]
            An \emph{event model} of $ \Lang $ is a tuple $ \E = (E, Q, \mathit{pre}, \mathit{post}) $ where:
                1) $E \neq \varnothing$ is a finite set of \emph{events}; 
                2) $Q: \agentSet \rightarrow 2^{E \times E}$ maps to each agent $i$ an accessibility relation $Q_i$;
                3) $\mathit{pre}: E \rightarrow \Lang$ maps to each event a \emph{precondition}; 
                4) $\mathit{post}: E \times \atomSet \rightarrow \Lang$ maps to each event-atom pair a \emph{postcondition}.
            \noindent An \emph{action} of $\Lang$ is a pair $\alpha = (\E, e)$ where $e \in E$ is the \emph{actual event}.
        \end{definition}

        \noindent 
        We often write $e Q_i f$ for $(e, f) \in Q_i$.
        We now introduce the notion of modal depth of an action (sequence), which will play a central role in our algorithm.
        \begin{definition}[Modal depth]\label{def:action-modal-depth}
            The \emph{modal depth} of formulas of $\mathcal{L}$ is defined by:
            $md(p) = 0$ (for $p \in \atomSet$), $md(\neg \phi) = md(\phi)$, $md(\phi_1 \land \phi_2) = \max\{md(\phi_1), md(\phi_2)\}$ and $md(\B{i} \phi) = 1 + md(\phi)$.
            The \emph{modal depth} of an action $\alpha$ is the maximal modal depth of its pre- and postconditions, \ie  $\md(\alpha) = \max\{\md(\mathit{pre}(e)), \md(\mathit{post}(e, p)) \mid e\in E, p \in \atomSet\}$.
            The \emph{modal depth} of $\pi = \alpha_1, \dots, \alpha_l$ is $\md(\pi) = \sum_{i \leq l} \md(\alpha_i)$.
        \end{definition}

 \label{sec:product-update}
        An action is executed in a state via the \emph{product update}.
        %
        \begin{definition}[Product Update]\label{def:update_em}
            Let $s = ((W, R, L), w)$ be a state and  
            $\alpha = ((E, Q, \mathit{pre}, \mathit{post}), e)$ an action. We say that $\alpha$ is \emph{applicable} in $s$ if $s \models \mathit{pre}(e)$, and if so, the \emph{product update} of $s$ with $\alpha$ is the state $s \otimes \alpha = ((W', R', L'), (w,e))$: \\[0.2mm]
                   \indent $W'	        = \{(w, e) \in W {\times} E \mid (M, w) \models \mathit{pre}(e)\}$ \\[0.2mm]
                   \indent
                    $R'_i	    = \{((w, e), (v, f)) \in W' {\times} W' \mid w R_i v \text{ and } e Q_i f\}$ \\[0.2mm]
                    \indent
                    $L'((w, e)) = \{p \in \atomSet \mid (M, w) \models \mathit{post}(e, \atom{p})\} $. 
        \end{definition}
        \begin{example} 
            \label{ex:update}
            \label{ex:action}
            We can define the \emph{public announcement} of a formula $\phi$ as the action $\mathit{ann}(\phi) = ((\{e\}, Q, \mathit{pre}, \mathit{post}),e)$ where $Q_i = \{(e, e)\}$ for all $i$, $\mathit{pre}(e) = \phi$ and $\mathit{post}(e, p) = p$ for all $p \in \atomSet$~\cite{conf/tark/Batlag1998}.
            Continuing Example \ref{ex:state}, 
            $\alpha = \mathit{ann}(\bigwedge_{0 \leq k \leq N} \neg \B{b} \mathit{has}(a, k))$ is the public announcement of ``$b$ doesn't know $a$'s number''.
            The state $s_0 \otimes \alpha$ is achieved by deleting $w_4$ from $s_0$: That world represents a situation where $b$ has $0$ and hence knows $a$ to have $1$, contradicting the announcement. 
        \end{example}

        We now extend the epistemic language $\Lang$ with \emph{dynamic modalities} $[\alpha]\phi$, where $\alpha = (A, e)$ is an action.
        We call the extended language $\Lang_\mathit{dyn}$, and we read $[\alpha]\phi$ as ``every execution of $\alpha$ yields a state satisfying $\phi$''.
        The semantics is:  

        {\centering
            $\begin{array}{@{}lll}
                s \models [\alpha]\phi & \text{ iff } & s \models \mathit{pre}(e) \text{ implies } s \otimes \alpha \models \phi
            \end{array}$
        \par}

        \noindent
        We often write $[\alpha_1, \dots, \alpha_l]\phi$ for $[\alpha_1]\dots[\alpha_l]\phi$ and define the \emph{modal depth} of $[\alpha]\phi$ as $\md([\alpha]\phi) = \md(\alpha) + \md(\phi)$.

    \paragraph{Epistemic Planning}\label{sec:plan_ex}
        We recall the notions of epistemic planning tasks and solutions~\cite{conf/ijcai/Aucher2013}.
        For a sequence $\pi = \alpha_1, \dots, \alpha_l$ of actions and $1 \leq k \leq l$, $\pi_{\leq k}$ denotes the prefix $\alpha_1, \dots, \alpha_k$ of $\pi$, and $s \otimes \pi$ the state $s \otimes \alpha_1 \dots \otimes \alpha_l$ (if $\pi$ is empty, this is just $s$).
        We say that $\pi$ is \emph{applicable} in $s$ if for all $k$, $\alpha_k$ is applicable in $s \otimes \pi_{\leq k-1}$.

        \begin{definition}\label{def:planning_task}
            \label{def:solution}
            An \emph{(epistemic) planning task} is a triple $T = (s_0, \actionSet,$ $ \phi_g)$, where $s_0$ is a state (the \emph{initial state}), $\actionSet$ is a finite set of actions, and $\phi_g \in \Lang$ is the \emph{goal formula}.
            A \emph{solution} (or \emph{plan}) to $T$ is a finite sequence $\pi = \alpha_1, \dots, \alpha_l$ of actions of $\actionSet$ such that $\pi$ is applicable in $s_0$ and $s_0 \models [\pi]\phi_g$.
        \end{definition}


    \paragraph{Bisimulations}
        We recall bisimulations and bounded bisimulations~\cite{book/cup/Blackburn2001}. 
        \begin{definition}\label{def:b-bisim}
            Let $b \geq 0$.
            A \emph{$b$-bisimulation} between states $s = ((W, R, L), w)$ and $s' = ((W', R', L'), w')$ is a sequence $Z_b \subseteq \cdots \subseteq Z_0 \subseteq W \times W'$ s.t.\ $(w,w') \in Z_b$ and for all $h < b$:
            \begin{compactitem}
                \item[{\quad [atom]}] If $(w, w') \in Z_0$, then $L(w) = L(w')$.
                \item[{\quad [forth$_h$]}] If $(w, w') \in Z_{h+1}$ and $w R_i v$, then there exists $v' \in W'$ such that $w' R'_i v'$ and $(v, v') \in Z_h$.
                \item[{\quad [back$_h$]}] If $(w, w') \in Z_{h+1}$ and $w' R'_i v'$, then there exists $v \in W$ such that $w R_i v$ and $(v, v') \in Z_h$.
            \end{compactitem}
            If a $b$-bisimulation between $s$ and $s'$ exists, $s$ and $s'$ are \emph{$b$-bisimilar}, denoted $s \bisim_b s'$.
            When $(M, w) \bisim_b (M', w')$, we often simply write $w \bisim_b w'$, and call $w$ and $w'$ \emph{$b$-bisimilar} (when $M,M'$ are clear from the context).
            The \emph{$b$-bisimulation class} of a world $w$ is $\class{w}{b} = \{v \in W \mid v \bisim_b w\}$.
           We say that $s$ and $s'$ are \emph{bisimilar}, denoted $s \bisim s'$, if there exists a single relation $Z$ satisfying the conditions above with the subscript on $Z$ removed everywhere. 
        \end{definition}
        Note that $s \bisim_b s'$ implies $s \bisim_{h} s'$ for all $h < b$ (similarly for $w \bisim_b w'$). 
        Given a set $\Phi \subseteq \Lang$ and states $s,s'$, 
        we say that $s$ and $s'$ \emph{agree on} $\Phi$ if, for all $\phi \in \Phi$, $s \models \phi$ iff $s' \models \phi$. 
        \begin{proposition}[\citeauthor{book/cup/Blackburn2001}~\citeyear{book/cup/Blackburn2001}]\label{prop:b-bisim}\label{prop:bisim}
            Two finite states are bisimilar iff they agree on $\Lang$. They are $b$-bisimilar iff they agree on $\{\phi \in \Lang \mid \md(\phi) \leq b\}$.
        \end{proposition}

        \begin{proposition}[\citeauthor{book/springer/vanDitmarsch2007}~\citeyear{book/springer/vanDitmarsch2007}]\label{prop:bisim-product}
            Let $s \bisim s'$ and let $\alpha$ be applicable in both $s$ and $s'$.
            Then, $s \otimes \alpha \bisim s' \otimes \alpha$.
        \end{proposition}
       
The following relatively recent result will play a central role in our paper. Together with Proposition~\ref{prop:b-bisim}, it gives a lower bound on the modal depth of formulas whose truth is preserved after product updates (after action executions). 
        \begin{proposition}[\citeauthor{journal/logcom/Bolander2022}~\citeyear{journal/logcom/Bolander2022}]\label{prop:b-bisim-product}\label{prop:b-bisim-product-sequence}
            Let $s \bisim_b s'$ and let $\pi$ be an action sequence applicable in both $s$ and $s'$ with $\md(\pi) \leq b$.
            Then, $s \otimes \pi \bisim_{b-\md(\pi)} s' \otimes \pi$.
        \end{proposition}
The cited source only formulates the result for  individual actions (action sequences of length 1), but it generalizes immediately. 
We now clarify the relevance of these results.
The modal depth of formulas formalizes the informal notion of reasoning depth used in the introduction: e.g., reasoning on the statement ``$a$ knows that $b$ knows neither agent has 0'' requires a formula modal depth 2.
Proposition~\ref{prop:b-bisim} shows that $b$-bisimilar states agree on all formulas up to modal depth $b$, so an agent limited to reasoning depth $b$ cannot distinguish between them.
Thus, the agent's internal state could be anyone of a given $b$-bisimulation class, e.g., a minimal one (see next section for an elaboration on this).
However, reasoning depth is not necessarily preserved under product updates.
For example, if $s_0'$ is obtained from $s_0$ (Figure~\ref{fig:state}) by deleting $w_3$ and $w_4$, then $s_0 \bisim_1 s_0'$, since formulas of modal depth $\leq 1$ are evaluated only in the actual world $w_1$ and its immediate neighbours, but
%
%
$s_0 \otimes \alpha \bisim_1 s_0' \otimes \alpha$ does \emph{not} hold when $\alpha$ is the public announcement that ``$b$ doesn't know $a$'s number'' of Example~\ref{ex:update}. This action will remove the rightmost world no matter whether it is applied to $s_0$ or $s_0'$, so when applied to $s_0'$, it will remove $w_2$, and then e.g.\ $\lozenge_a \mathit{has}(a, 3)$ will be true in $s_0 \otimes \alpha$ but not in $s_0' \otimes \alpha$. So if an agent needs to be able to reason to depth 1 about the state resulting from executing $\alpha$, it needs to be able to reason to at least to depth 2 about the state before the action. Proposition~\ref{prop:b-bisim-product} then guarantees that reasoning to depth 2 about the former state suffices, since $\md(\alpha) =1$. This is key in depth-bounded epistemic planning, since 
planning is reasoning about possible future states, and 
depth-bounded reasoning to depth $b$ corresponds to $b$-bisimulation invariance.
\begin{proposition}\label{prop:b-bisim-dyn} 
    If two finite states $s$ and $s'$ are $b$-bisimilar, they agree on $\{\phi \in \Lang_\mathit{dyn} \mid \md(\phi) \leq b\}$.\footnote{Full proofs are available in the Appendix of the extended version available at \url{https://arxiv.org/abs/2406.01139}.}
\end{proposition}

\ifthenelse{\boolean{expandedVersion}}{
    \begin{proof}
        Suppose $s$ and $s'$ are $b$-bisimilar. We need to prove that they agree on every formula $\phi \in \Lang_\mathit{dyn}$ with $\md(\phi) \leq b$. 
        The proof is by induction on the structure of $\phi$. The cases where $\phi$ is $\psi \land \gamma$, $\neg \psi$ or $\Box_i \psi$ are as in the proof of Proposition~\ref{prop:b-bisim}. Only remaining is the case of $\phi = [\alpha]\psi$ with $\alpha = (A,e)$ and $\md([\alpha]\psi) \leq b$. Note that, by definition, $\md([\alpha]\psi) = \md(\alpha) + \md(\psi)$. Thus $\md(\alpha) \leq b$, implying that $s$ and $s'$ agree on $pre(e)$ (Proposition~\ref{prop:b-bisim}). If $pre(e)$ holds in neither, $\phi$ holds in both, and we are done.  If $pre(e)$ holds in both, we need to prove that $s \otimes \alpha$ and $s' \otimes \alpha$ agree on  $\psi$. Proposition~\ref{prop:b-bisim-product-sequence} gives us $s' \otimes \alpha \bisim_{b - \md(\alpha)} s \otimes \alpha$, and since $\md(\psi) = \md([\alpha]\psi) - \md(\alpha) \leq b - \md(\alpha)$, the induction hypothesis implies the required. 
    \end{proof}
}{
    \begin{proof}[Proof sketch]
        The proof is by induction on the structure of $\phi$, where the cases of the connectives and modalities in $\Lang$ are exactly as in the proof of Proposition~\ref{prop:b-bisim}, and the case of the dynamic modalities $[\alpha]\psi$ follows from Proposition~\ref{prop:b-bisim-product-sequence} and the definition of $\md([\alpha]\psi)$.
    \end{proof}
}

    \section{Rooted and Canonical $b$-Contractions}\label{sec:canonical-contractions}
A significant challenge for epistemic planners is to 
handle the rapid growth of the size of states following updates~\cite{journals/jancl/Bolander2011}, even when minimizing states using bisimulation contractions~\cite{conf/ijcai/Yu2013,conf/icaps/Fabiano2020,conf/kr/BolanderDH21}. Here we seek to model agents with a bound on their reasoning depth (a bound on the modal depth of the formulas they can reason with).  
We achieve this by using  $b$-bisimulation contractions, abbreviated \emph{$b$-contractions}. 
These were introduced to epistemic planning by   
\citeauthor{conf/ijcai/Yu2013}~(\citeyear{conf/ijcai/Yu2013}), 
defining the $b$-contraction $\contr{s}{b}$ of a  
state $s = ((W, R, L), w)$ 
as the quotient structure of $s$ wrt.\ $\bisim_b$, \ie, $\contr{s}{b} = ((W', R', L'), \class{w}{b})$ with $W' = \{\class{w}{b} \mid w \in W\}$, $R'_i = \{(\class{w}{b}, \class{v}{b}) \mid w R_i v\}$, and $L'(\class{w}{b}) = L(w)$.
We call this the \emph{standard $b$-contraction} of $s$.  Unfortunately, $\contr{s}{b}$ is not always minimal (\ie it does not have the smallest number of worlds and edges) among the states $b$-bisimilar to $s$, as the next example shows~\cite{conf/aiml/Bolander2024}.

\begin{figure}
    \centering
    \scalebox{0.9}{
\begin{tikzpicture}[>=latex,auto,style={font=\sffamily\footnotesize}]
    \node[dworld, label={below:$w_{b}{:}p$}, label={left:$s = \contr{s}{b} =$}]
                                                     (w0)  at (0,   0) {};
    \node[world,        label={below:$w_{b-1}{:}p$}] (w1)  at (1.2, 0) {};
    \node[world,        label={below:$w_1{:}p$}]     (w2)  at (2.4, 0) {};
    \node[world,        label={below:$w_0{:}p$}]     (w3)  at (3.6, 0) {};

    \path
        (w0) edge[->]        (w1)
        (w1) edge[->,dotted] (w2)
        (w2) edge[->]        (w3)
    ;
    
    \node[dworld, label={below:$w'_{b}{:}p$}, label={left:$\rootedContr{s}{b}{<}=$}] (v0) at (6.5, 0) {};

    \path
        (v0) edge[->, loop, out=50, in=130, looseness=14] (v0)
    ;
\end{tikzpicture}
}
    \caption{Standard ($\contr{s}{b}$) and rooted ($\rootedContr{s}{b}{<}$) $b$-contractions of $s$ (Definition \ref{def:rooted-b-contr}).
    As we have only one agent, we omit agent labels.
    }
    \label{fig:chain}
\end{figure}
\begin{example} 
    \label{ex:chain}
    Consider the chain state $s$ in Figure \ref{fig:chain}, left.
    Since $p$ is true in all worlds, and the length of the chain is $b$, a minimal state $b$-bisimilar to $s$ is a singleton state with a loop (Figure \ref{fig:chain}, right). This is because the loop state preserves all formulas up to depth $b$, cf.\ Proposition~\ref{prop:b-bisim}. However, the standard $b$-contraction of $s$ is simply $s$ itself, as no two worlds of $s$ satisfy the same formulas up to depth $b$. 
\end{example}
\noindent In the following, $b \geq 0$ denotes a constant (a \emph{bound}), $s = (M, w_d)$ a state with $M = (W, R, L)$, $<$ a total order on $W$.
\myparagraph{Rooted $b$-contractions}
Recently, \citeauthor{conf/aiml/Bolander2024}~\shortcite{conf/aiml/Bolander2024} developed a novel type of $b$-contractions, called \emph{rooted $b$-contractions}, that guarantee minimality of contracted states. We briefly state the definition and minimality result from that paper (Definitions~\ref{def:bound}--\ref{def:rooted-b-contr} and Theorem~\ref{th:b-contr} below), and refer the reader to the paper for further details. 

\begin{definition}[Depth and Bound]\label{def:bound}
        \label{def:depth}
        The \emph{depth} $\depth{w}$ of a world $w \in W$ is the length of the shortest path from $w_d$ to $w$ ($\infty$ if no such path exists).
        The \emph{bound} of $w$ is $\bound{w} = b - \depth{w}$. 
    \end{definition}
\begin{definition}
    \label{def:representative}\label{def:max-repr}\label{def:repr-class}
    \label{def:least-repr}
    Let $x,y \in W$ with $b(x), b(y) \geq 0$. We say that $x$ \emph{represents} $y$, denoted $x \succeq y$, if $\bound{x} \geq \bound{y}$ and $x \bisim_{\bound{y}} y$. If furthermore $\bound{x} > \bound{y}$, we say that $x$ \emph{strictly represents} $y$, denoted by $x \succ y$. 
    The set of \emph{maximal representatives} of $W$ is the set $\maxRepr{W} = \{ x \in W \mid \bound{x} \geq 0 \textnormal{ and } \neg \exists y (y \succ x) \}$.
    We also let $\strictMaxRepr{W} = \{ x \in \maxRepr{W} \mid \bound{x} > 0 \}$.
    For $h \leq b$,
    the \emph{least $h$-representative} of $w$ 
    is the world 
    $\leastRepr{w}{h} = \min_< \{ v \in \maxRepr{W} \mid v \bisim_h w \}$. 
    The \emph{representative class} of $w \in W$ is $\class{w}{\bound{w}}$, compactly denoted $\reprClass{w}$.
\end{definition}


\begin{definition}
    \label{def:rooted-b-contr}
    The \emph{rooted $b$-contraction} of $s$ 
    is $\rootedContr{s}{b}{<} = ((W', R', L'), \reprClass{w_d})$, where: \\[0.2mm]
     \indent $W'    = \{\reprClass{x} \mid x \in \maxRepr{W}\}$ \\[0.2mm]
     \indent $R'_i  = \{(\reprClass{x}, \reprClass{\leastRepr{y}{\bound{x}{-}1}}) \mid x \in \strictMaxRepr{W} \text{ and } x R_i y$\} \\[0.2mm]
    \indent
       $L'(\reprClass{x}) = L(x)$, for all $\reprClass{x} \in W'$.
\end{definition}


\begin{theorem}
    [\citeauthor{conf/aiml/Bolander2024}~\citeyear{conf/aiml/Bolander2024}]
    \label{th:b-contr}
$\rootedContr{s}{b}{<}$ is a minimal state $b$-bisimilar to $s$, \ie, it has a minimal number of worlds and edges among all states $b$-bisimilar to $s$. 
\end{theorem}



\myparagraph{Canonical $b$-contractions}
   Below, we devise a novel planning algorithm using \emph{graph search}~\cite{russell2010artificial} to find plans. As we want to limit the reasoning depth of the planning agent to some bound $b$, we can replace each state $s$ in the search tree by $\rootedContr{s}{b}{<}$, since    
   Proposition~\ref{prop:b-bisim} and Theorem~\ref{th:b-contr} together give us that $\rootedContr{s}{b}{<}$ is a minimal state preserving the truth of formulas up to depth $b$ in $s$. 
   Due to the minimality of rooted $b$-contractions, this would guarantee minimality of the representation of each state in the search tree. 
   However, when using graph search for planning, we also need to be able to efficiently check whether a computed state is already in the search tree. In our case, this would amount to checking $b$-bisimilarity between the computed state and existing states in the search tree, which is costly.     
   \citeauthor{conf/kr/BolanderDH21}~\shortcite{conf/kr/BolanderDH21}
   found a solution to this problem in the case of standard bisimulations by 
   devising a so-called \emph{ordered partition refinement} algorithm for computing bisimulation contractions guaranteeing two bisimilar states to have \emph{identical} contractions, hence reducing bisimilarity checks to identity checks. 
   Inspired by this work, we 
  define a notion of \emph{canonical $b$-contractions} guaranteeing two $b$-bisimilar states to have identical contractions, in turn providing fast state comparisons for our planning algorithm. 
Before giving the definition, we illustrate the issue of rooted $b$-contractions with an example. 

    \begin{figure*} 
        \centering
        \hfill
        \begin{minipage}[b]{0.2\textwidth}
            \begin{tikzpicture}[>=stealth',node distance=1.7em and 1.5em,style={font=\sffamily\footnotesize}, baseline=(wd),label distance=0.25em]
    \begin{scope}[inner sep=0.5pt]
        \node[dworld,                    label={right:$ w_0{:}p$}] (wd) {};
        \node[world,  below left =of wd, label={left :$ w_1{:}q$}] (w1) {};
        \node[world,  below right=of wd, label={right:$~w_2{:}q$}] (w2) {};
        \node[world,  below=of w1,       label={left :$ w_3{:}r$}] (w3) {};
        \node[world,  below=of w2,       label={right:$ w_4{:}r$}] (w4) {};
    \end{scope}

    \path
        (wd) edge[-latex]  (w1)
        (wd) edge[-latex]  (w2)
        (w1) edge[latex-latex] (w3)
        (w2) edge[-latex]  (w4)
        (w3) edge[-latex]  (w2)
        (w2) edge[-latex, loop, out=110, in=30, looseness=14] (w2)
    ;
 \end{tikzpicture}
            \caption{State $s$.}\label{fig:ex-model-1}
        \end{minipage}
        \hfill
        \begin{minipage}[b]{0.2\textwidth}
            \begin{tikzpicture}[>=stealth',node distance=1.7em and 1.5em,style={font=\sffamily\footnotesize}, baseline=(wd),label distance=0.25em]
    \begin{scope}[inner sep=0.5pt]
        \node[dworld,                    label={right:$ w_0{:}p$}] (wd) {};
        \node[world,  below left =of wd, label={left :$ w_2{:}q$}] (w1) {};
        \node[world,  below right=of wd, label={right:$~w_1{:}q$}] (w2) {};
        \node[world,  below=of w1,       label={left :$ w_3{:}r$}] (w3) {};
        \node[world,  below=of w2,       label={right:$ w_4{:}r$}] (w4) {};
    \end{scope}

    \path
        (wd) edge[-latex]  (w1)
        (wd) edge[-latex]  (w2)
        (w1) edge[latex-latex] (w3)
        (w2) edge[-latex]  (w4)
        (w3) edge[-latex]  (w2)
        (w2) edge[-latex, loop, out=110, in=30, looseness=14] (w2)
    ;
 \end{tikzpicture}
            \caption{State $t$.}\label{fig:ex-model-2}
        \end{minipage} 
        \hfill
        \begin{minipage}[b]{0.2\textwidth} 
            \begin{tikzpicture}[>=latex,node distance=1.7em and 1.5em,style={font=\sffamily\footnotesize}, baseline=(wd)]
    \begin{scope}[inner sep=0.5pt]
        \node[dworld,                    label={right:$~ \class{w_0}{3}{:}p$}] (wd) {};
        \node[world,  below left =of wd, label={left :$  \class{w_1}{2}{:}q$}] (w1) {};
        \node[world,  below right=of wd, label={right:$~~\class{w_2}{2}{:}q$}] (w2) {};
        \node[world,  below=of w1,       label={left :$  \class{w_3}{1}{:}r$}] (w3) {};
        \node[world,  below=of w2,       label={right:$  \class{w_4}{1}{:}r$}] (w4) {};
    \end{scope}

    \path
        (wd) edge[->]  (w1)
        (wd) edge[->]  (w2)
        (w1) edge[<->] (w3)
        (w2) edge[->]  (w4)
        (w2) edge[->, loop, out=110, in=30, looseness=14] (w2)
    ;
\end{tikzpicture}
            \caption{State $\rootedContr{s}{3}{<}$.}\label{fig:non-idempotent-1}
        \end{minipage} 
        \hfill
        \begin{minipage}[b]{0.2\textwidth}
            \begin{tikzpicture}[>=latex,node distance=1.7em and 1.5em,style={font=\sffamily\footnotesize}, baseline=(wd)]
    \begin{scope}[inner sep=0.5pt]
        \node[dworld,                    label={right:$~ \class{w_0}{3}{:}p$}] (wd) {};
        \node[world,  below left =of wd, label={left :$  \class{w_2}{2}{:}q$}] (w1) {};
        \node[world,  below right=of wd, label={right:$~~\class{w_1}{2}{:}q$}] (w2) {};
        \node[world,  below=of w1,       label={left :$  \class{w_3}{1}{:}r$}] (w3) {};
        \node[world,  below=of w2,       label={right:$  \class{w_4}{1}{:}r$}] (w4) {};
    \end{scope}

    \path
        (wd) edge[->]  (w1)
        (wd) edge[->]  (w2)
        (w1) edge[->]  (w3)
        (w2) edge[->]  (w4)
        (w3) edge[->]  (w2)
        (w2) edge[->, loop, out=110, in=30, looseness=14] (w2)
    ;
\end{tikzpicture}
            \caption{State $\rootedContr{t}{3}{<}$.}\label{fig:non-idempotent-2}
        \end{minipage}
        \hfill\hfill
        \label{fig:non-idempotent}
    \end{figure*}
    \begin{example}\label{ex:non-isomorphic}
        Let $s$ be the state in Figure \ref{fig:ex-model-1},  
        let $t$ be a state only differing from $s$ by swapping the names of $w_1$ and $w_2$ (Figure~\ref{fig:ex-model-2}),
        let the total order $<$ on $W$ (the world set of $s$ and $t$) be given by $w_i < w_j$ iff $i < j$, and let $b = 3$.
        Then $b(w_0) = 3$,  $b(w_1) = b(w_2) = 2$, and $b(w_3) = b(w_4) = 1$.
        Since in both states no two worlds are $1$-bisimilar, all worlds of $W$ are maximal representatives,
        so $\rootedContr{s}{3}{<}$ and $\rootedContr{t}{3}{<}$ (Figures~\ref{fig:non-idempotent-1} and \ref{fig:non-idempotent-2}) will contain a world $\reprClass{w} = \class{w}{b(w)}$ for each $w \in W$.
        Note that both $3$-contracted states only differ from their original model by the deletion of a different ``unnecessary'' edge (as this depends on $<$).
        Hence, despite $s$ and $t$ being isomorphic, $\rootedContr{s}{3}{<}$ and $\rootedContr{t}{3}{<}$ are not!
        This shows that rooted $b$-contractions do not ensure unique representatives of classes of $b$-bisimilar states. 
   
    \end{example}

    \begin{definition}\label{def:h-signature}
        The \emph{$h$-signature} $\Sign{w}{h}$ of a world $w$ is the pair $(L(w), \Sigma_h(w))$, where the function $\Sigma_h(w)$ maps to each agent $i$ a set $\SignMap{w}{i}{h} = \{\Sign{v}{h-1}\ |\ w R_i v\}$ of $(h{-}1)$-signatures, if $h > 0$, and $\varnothing$ otherwise.
        %
        We call $\Sign{w}{\bound{w}}$ the \emph{representative signature} of $w$, compactly denoted $\reprSign{w}$.
    \end{definition}
    In the $h$-signature of $w$,
    the label $L(w)$ describes what atoms hold at $w$ and an $(h{-}1)$-signature $\Sign{v}{h-1} \in \SignMap{w}{i}{h}$ represents agent $i$'s knowledge/beliefs at $v$ up to depth $h{-}1$.
    All $(h{-}1)$-signatures map to each agent a set of $(h{-}2)$-signatures, and so on until $0$-signatures, which only contain worlds labels. 
    Hence, the $h$-signature of $w$ captures the model structure up to depth $h$ from $w$.
    This will allow us to prove that two worlds are $h$-bisimilar iff they have the same $h$-signature 
    (Lemma \ref{th:same-signature}).
    In what follows, $h \geq 0$ denotes a constant, and $M = (W, R, L)$ and $M' = (W', R', L')$ two models with $x \in W$ and $x' \in W'$.

    \ifthenelse{\boolean{expandedVersion}}{
        \begin{lemma}\label{lem:h-signature}
            If $\Sign{x}{h+1} = \Sign{x'}{h+1}$, then $\Sign{x}{h} = \Sign{x'}{h}$.
        \end{lemma}
        \begin{proof}
            By induction on $h$. Base case ($h=0$):
            from $\Sign{x}{1} = \Sign{x'}{1}$ we get $L(x) = L'(x')$ and, thus, $\Sign{x}{0} = \Sign{x'}{0}$. Induction step ($h >0$): 
            For all $y \in W, y' \in W'$, suppose $\Sign{y}{h} = \Sign{y'}{h}$ implies $\Sign{y}{h-1} = \Sign{y'}{h-1}$.
            Let $\Sign{x}{h+1} = \Sign{x'}{h+1}$.
            Then $L(x) = L'(x')$, and it is only left 
            to prove that $\SignMap{x}{i}{h} = \SignMap{x'}{i}{h}$ for all $i \in \agentSet$.
            We only show $\SignMap{x}{i}{h} \subseteq \SignMap{x'}{i}{h}$, the other direction being symmetrical.
            Letting $z \in \SignMap{x}{i}{h}$, we need to show $z \in \SignMap{x'}{i}{h}$. Since $z \in \SignMap{x}{i}{h}$, we have $z = \sigma_{h-1}(y)$ for some $y$ with $x R_i y$. From $x R_i y$ we get $\Sign{y}{h} \in \SignMap{x}{i}{h+1}$, which, since $\sigma_{h+1}(x) = \sigma_{h+1}(x')$, implies $\Sign{y}{h} \in \SignMap{x'}{i}{h+1}$. From $\Sign{y}{h} \in \SignMap{x'}{i}{h+1}$ we get that there is a $y'$ s.t.\ $x' R'_i y'$ and $\Sign{y}{h} = \Sign{y'}{h}$. 
            By i.h.\ we now have $z = \Sign{y}{h-1} = \Sign{y'}{h-1}$, and since $x' R'_i y'$, we then also get $z \in \SignMap{x'}{i}{h}$, as required.
        \end{proof}
    }{
    }

    \begin{lemma}\label{th:same-signature}
        $x \bisim_h x'$ iff $\Sign{x}{h} = \Sign{x'}{h}$.
    \end{lemma}

    \ifthenelse{\boolean{expandedVersion}}{
        \begin{proof}
            ($\Rightarrow$)
            By induction on $h$. Base case ($h = 0$): $x \bisim_0 x'$ iff $L(x) = L'(x')$ iff $\Sign{x}{0} = \Sign{x'}{0}$. Induction step ($h > 0$):
            Suppose $y \bisim_{h-1} y'$ implies $\Sign{y}{h-1} = \Sign{y'}{h-1}$, for all $y \in W$, $y' \in W'$.
            Letting $x \bisim_h x'$, we prove that $\Sign{x}{h} = \Sign{x'}{h}$.
            From $x \bisim_h x'$, [atom] of Definition \ref{def:b-bisim} gives $L(x) = L'(x')$.
            Only left to show is that $\SignMap{x}{i}{h} = \SignMap{x'}{i}{h}$ for all $i \in \agentSet$.
            We only show $\SignMap{x}{i}{h} \subseteq \SignMap{x'}{i}{h}$, the other direction being symmetrical. Letting $z \in \SignMap{x}{i}{h}$, we need to show $z \in \SignMap{x'}{i}{h}$. From $z \in \SignMap{x}{i}{h}$ we get $z = \Sign{y}{h-1}$ for some $y$ with $x R_i y$. 
            Since $x \bisim_h x'$ and $x R_i y$, there exists $y' \in W'$ such that $x' R'_i y'$ and $y \bisim_{h-1} y'$.
            By induction hypothesis, this implies $\Sign{y}{h-1} = \Sign{y'}{h-1}$. We now have $z = \Sign{y}{h-1} = \Sign{y'}{h-1}$, and since $x' R'_i y'$, we then also have $z \in \SignMap{x'}{i}{h}$, as required.

            ($\Leftarrow$)
            Let $\Sign{x}{h} = \Sign{x'}{h}$. For all $g \leq h$, let $Z_g = \{(w, w') \in W \times W' \mid \Sign{w}{g} = \Sign{w'}{g}\}$.
            We show that $Z_h, \dots, Z_0$ is an $h$-bisimulation between $x$ and $x'$. Clearly $(x,x') \in Z_h$. 
            That $Z_h \subseteq \dots \subseteq Z_0$ follows from Lemma \ref{lem:h-signature}; and [atom] follows from the definition of $\sigma_0$.
            We only show [forth$_h$] ([back$_h$] being symmetric).
            Let $g < h$, $(w, w') \in Z_{g+1}$ and $w R_i v$.
            We need to find $v' \in W'$ such that $w' R'_i v'$ and $(v, v') \in Z_g$.
            From $w R_i v$ we get $\Sign{v}{g} \in \SignMap{w}{i}{g+1}$.
            Since $(w, w') \in Z_{g+1}$, we get $\Sign{w}{g+1} = \Sign{w'}{g+1}$, implying $\Sign{v}{g} \in \SignMap{w'}{i}{g+1}$. Thus there exists $v'$ with $w' R_i' v'$ and $\Sign{v}{g} = \Sign{v'}{g}$, implying $(v, v') \in Z_g$.
        \end{proof}
    }{
        \begin{proof}[Proof sketch]
            ($\Rightarrow$)
            Induction on $h$.
            For $h = 0$, $x \bisim_0 x'$ iff $L(x) = L'(x')$ iff $\Sign{x}{0} = \Sign{x'}{0}$.
            For $h > 0$, $x \bisim_h x'$ implies $L(x) = L'(x')$, and for all $i \in \agentSet$, the sets of $(h{-}1)$-signatures of $i$-accessible worlds are equal (i.h.\ and Definition~\ref{def:h-signature}).
            ($\Leftarrow$)
            Assume $\Sign{x}{h} = \Sign{x'}{h}$.
            Letting $Z_g = \{(w, w') \mid \Sign{w}{g} = \Sign{w'}{g}\}$ for $g \leq h$, $Z_h, \dots, Z_0$ is an $h$-bisimulation between $x$ and $x'$ (i.h.\ and Definition~\ref{def:h-signature}).
        \end{proof}
    }

    As $\atomSet$ and $\agentSet$ are fixed finite sets, we can assume a fixed total order on them.
    This induces a fixed total order $\lessdot$ on signatures as in \citeauthor{conf/kr/BolanderDH21}~(\citeyear{conf/kr/BolanderDH21}).
    From this, we define the notion of \emph{canonical signatures}.

    \begin{definition}\label{def:canonical-signature}
        The \emph{canonical signature to depth $h$} ($h \leq b$) of a world $w \in W$ is the representative signature $\canonicalSign{w}{h} = \min_\lessdot \{ \reprSign{v} \mid v \in \maxRepr{W} \text{ and } \Sign{v}{h} = \Sign{w}{h} \}$.
    \end{definition}


    \begin{definition}\label{def:canonical-b-contr}
        The \emph{canonical $b$-contraction} of $s$ is the state $\canonicalContr{s}{b} = ((W', R', L'), \reprSign{w_d})$, where:\\[0.2mm]
        \indent $W'    = \{\reprSign{x} \mid x \in \maxRepr{W}\}$\\[0.2mm]
        \indent $R'_i  = \{(\reprSign{x}, \canonicalSign{y}{b(x){-}1}) \mid x \in \strictMaxRepr{W} \text{ and } x R_i y\}$\\[0.2mm]
        \indent $L'(\reprSign{x}) = L(x)$, for all $\reprSign{x} \in W'$.
    \end{definition}
    Note that canonical $b$-contractions only differ from rooted $b$-contractions by the naming of worlds ($\reprSign{x}$ instead of $\reprClass{x}$) and by the choice of representative of the class of worlds that are $(b(x){-}1)$-bisimilar to $y$ in the definition of $R'_i$ (using $\canonicalSign{y}{b(x){-}1}$ instead of $\reprClass{\leastRepr{y}{\bound{x}{-}1}}$). This means that the proof of Theorem~\ref{th:b-contr} carries directly over to this setting: 
    \begin{theorem}\label{th:canonical-b-contr}
        $\canonicalContr{s}{b}$ is a minimal state $b$-bisimilar to $s$.
    \end{theorem}
%
%
\ifthenelse{\boolean{expandedVersion}}{
Furthermore, as we will now see, two states are $b$-bisimilar if and only if they have identical canonical $b$-contractions. First we need a couple of technical lemmas.  
    \begin{lemma}[\citeauthor{conf/aiml/Bolander2024}~\shortcite{conf/aiml/Bolander2024}]\label{lem:bound}
        If $x R_i y$, then $\bound{y} \geq \bound{x}-1$.
    \end{lemma}

    \begin{lemma}\label{lem:bound-bisim}
        Let $s \bisim_b s'$.
        For any world $x$ in $s$ with $\bound{x} \geq 0$ there exists an $x'$ in $s'$ such that $x \bisim_{\bound{x}} x'$, and vice versa.
    \end{lemma}

    \begin{proof}
        Let $s = (M, w_d)$ and $s' = (M', w'_d)$ with $M = (W, R, L)$ and $M' = (W', R', L')$.
        Since $\bound{x} \geq 0$, by Definition \ref{def:bound}, there exists a finite path from $w_d$ to $x$ of length $\depth{x} = b-\bound{x}$.
        Since $s \bisim_b s'$, by repeated applications of the [forth$_h$] condition in Definition \ref{def:b-bisim}, we get that there exists a path of length $b-\bound{x}$ from $w'_d$ to a world $x' \in W'$ such that $x \bisim_{\bound{x}} x'$, as required.
        The vice versa follows by starting from a world $x' \in W'$ with $\bound{x'} \geq 0$ and by applying [back$_h$] instead of [forth$_h$].
    \end{proof}



    \setcounter{theorem}{2}
    \begin{theorem}\label{th:canonical-contractions}
        $s \bisim_b t$ iff $\canonicalContr{s}{b} = \canonicalContr{t}{b}$.
    \end{theorem}
    \begin{proof}
        ($\Leftarrow$)
        Directly follows by Theorem~\ref{th:canonical-b-contr}.
        ($\Rightarrow$)
        Let $s = (M, w_d)$, $t = (\mathtt{M}, \mathtt{w_d})$, $\canonicalContr{s}{b} = (M^\star, \reprSign{w_d})$ and $\canonicalContr{t}{b} = (\mathtt{M}^\star, \reprSign{\mathtt{w_d}})$, with $M = (W, R, L)$, $\mathtt{M} = (\mathtt{W}, \mathtt{R}, \mathtt{L})$, $M^\star = (W^\star, R^\star, L^\star)$ and $\mathtt{M}^\star = (\mathtt{W}^\star, \mathtt{R}^\star, \mathtt{L}^\star)$.
        We need to show that
        \begin{inparaenum}
            \item\label{itm:canonical-contr-1} $\reprSign{w_d} = \reprSign{\mathtt{w_d}}$,
            \item\label{itm:canonical-contr-2} $W^\star = \mathtt{W}^\star$,
            \item\label{itm:canonical-contr-3} $R^\star = \mathtt{R}^\star$, and
            \item\label{itm:canonical-contr-4} $L^\star = \mathtt{L}^\star$.
        \end{inparaenum}

        \ref{itm:canonical-contr-1}.
        Since $s \bisim_b t$, by Definition \ref{def:b-bisim} we get $w_d \bisim_b \mathtt{w_d}$.
        Since by Definition \ref{def:bound} we have $\bound{w_d} = \bound{\mathtt{w_d}} = b$, Lemma \ref{th:same-signature} then gives $\reprSign{w_d} = \reprSign{\mathtt{w_d}}$.

        \ref{itm:canonical-contr-2}.
        We now only show that $W^\star \subseteq \mathtt{W}^\star$, the other direction being symmetrical.
        Let $x^\star \in W^\star$.
        By Definition \ref{def:canonical-b-contr}, there exists a world $x \in \maxRepr{W}$ such that $x^\star = \reprSign{x}$.
        We need to show that $\reprSign{x} \in \mathtt{W}^\star$.
        We reason as follows.
        If there exists an $\mathtt{x} \in \maxRepr{\mathtt{W}}$ such that $x \bisim_{\bound{x}} \mathtt{x}$ and $\bound{x} = \bound{\mathtt{x}}$, then by Lemma \ref{th:same-signature} and Definition \ref{def:canonical-b-contr}, we get $\reprSign{x} = \reprSign{\mathtt{x}} \in \mathtt{W}^\star$.
        We now show that such an $\mathtt{x}$ always exists.
        Since $x \in \maxRepr{W}$, Definition \ref{def:max-repr} gives $\bound{x} \geq 0$.
        Then, by Lemma \ref{lem:bound-bisim} there exists a $\mathtt{y} \in \mathtt{W}$ with $x \bisim_{\bound{x}} \mathtt{y}$.
        Since by the proof of Lemma \ref{lem:bound-bisim} there is a path from $\mathtt{w_d}$ to $\mathtt{y}$ of length $b-\bound{x}$, we get $\depth{\mathtt{y}} \leq b-\bound{x}$, \ie $\bound{\mathtt{y}} \geq \bound{x}$.
        Let $\mathtt{x} \in \maxRepr{\mathtt{W}}$ be a maximal representative of $\mathtt{y}$.
        Then by Definition \ref{def:max-repr} we have $\bound{\mathtt{x}} \geq \bound{\mathtt{y}}$ and $\mathtt{x} \bisim_{\bound{\mathtt{y}}} \mathtt{y}$, which, since $\bound{\mathtt{y}} \geq \bound{x}$, implies $\mathtt{x} \bisim_{\bound{x}} \mathtt{y}$.
        Together with $x \bisim_{\bound{x}} \mathtt{y}$, we also get $x \bisim_{\bound{x}} \mathtt{x}$.
        So far we have found a world $\mathtt{x} \in \maxRepr{\mathtt{W}}$ such that $x \bisim_{\bound{x}} \mathtt{x}$.
        All is left to show is that $\bound{x} = \bound{\mathtt{x}}$.
        Since $\bound{\mathtt{x}} \geq \bound{\mathtt{y}}$ and $\bound{\mathtt{y}} \geq \bound{x}$, we get $\bound{\mathtt{x}} \geq \bound{x}$.
        We now show by contradiction that $\bound{\mathtt{x}} \leq \bound{x}$.
        Assume $\bound{\mathtt{x}} > \bound{x}$.
        By using Lemma \ref{lem:bound-bisim} as above, we get that there exists a $y \in W$ such that $\mathtt{x} \bisim_{\bound{\mathtt{x}}} y$ and $\bound{y} \geq \bound{\mathtt{x}}$.
        From $\bound{y} \geq \bound{\mathtt{x}}$, $\bound{\mathtt{x}} > \bound{x}$, $x \bisim_{\bound{x}} \mathtt{x}$ and $\mathtt{x} \bisim_{\bound{\mathtt{x}}} y$, we get $\bound{y} > \bound{x}$ and $y \bisim_{\bound{x}} x$, implying $y \succ x$, which contradicts $x \in \maxRepr{W}$.
        Thus, $\bound{\mathtt{x}} \leq \bound{x}$, as required.

        \ref{itm:canonical-contr-3}.
        Let $i \in \agentSet$.
        We only show that $R_i^\star \subseteq \mathtt{R}_i^\star$, the other direction being symmetrical.
        So let $(\reprSign{x}, \canonicalSign{y}{\bound{x}-1}) \in R_i^\star$.
        We need to show that $(\reprSign{x}, \canonicalSign{y}{\bound{x}-1}) \in \mathtt{R}_i^\star$.
        By Definition \ref{def:canonical-b-contr}, we have $x \in \maxRepr{W}$, $x R_i y$ and $\bound{x} > 0$.
        From the proof of Item \ref{itm:canonical-contr-2}, there exists $\mathtt{x} \in \maxRepr{\mathtt{W}}$ such that $x \bisim_{\bound{x}} \mathtt{x}$ and $\bound{x} = \bound{\mathtt{x}}$.
        Since $x \bisim_{\bound{x}} \mathtt{x}$ and $x R_i y$, there exists $\mathtt{y} \in \mathtt{W}$ such that $\mathtt{x} \mathtt{R}_i \mathtt{y}$ and $y \bisim_{\bound{x}-1} \mathtt{y}$.
        Since $\mathtt{x} \in \maxRepr{\mathtt{W}}$, $\mathtt{x} \mathtt{R}_i \mathtt{y}$ and $\bound{\mathtt{x}} = \bound{x} > 0$, by Definition \ref{def:canonical-b-contr} we get $(\reprSign{\mathtt{x}}, \canonicalSign{\mathtt{y}}{\bound{\mathtt{x}}-1}) \in \mathtt{R}_i^\star$.
        Since $x \bisim_{\bound{x}} \mathtt{x}$ and $\bound{x} = \bound{\mathtt{x}}$, Lemma \ref{th:same-signature} gives $\reprSign{x} = \reprSign{\mathtt{x}}$.
        It remains to show that $\canonicalSign{y}{\bound{x}-1} = \canonicalSign{\mathtt{y}}{\bound{\mathtt{x}}-1}$.
        Let $S$ and $\mathtt{S}$ be the sets such that $\canonicalSign{y}{\bound{x}-1} = \min_\lessdot S$ and $\canonicalSign{\mathtt{y}}{\bound{\mathtt{x}}-1} = \min_\lessdot \mathtt{S}$, as in Definition \ref{def:canonical-signature}.
        Then, it suffices to prove $S = \mathtt{S}$.
        We only show $S \subseteq \mathtt{S}$, the other direction being symmetrical.
        Let $\reprSign{z} \in S$.
        Then, $z \in \maxRepr{W}$ and $\Sign{z}{\bound{x}-1} = \Sign{y}{\bound{x}-1}$ (Definition \ref{def:canonical-signature}).
        Since $z \in \maxRepr{W}$, from the proof of Item \ref{itm:canonical-contr-2}, there exists $\mathtt{z} \in \maxRepr{\mathtt{W}}$ such that $z \bisim_{\bound{z}} \mathtt{z}$ and $\bound{z} = \bound{\mathtt{z}}$.
        We now show by contradiction that $\bound{z} \geq \bound{x}-1$.
        Assume $\bound{z} < \bound{x}-1$.
        Since $x R_i y$, Lemma \ref{lem:bound} gives $\bound{y} \geq \bound{x}-1$ and, thus, $\bound{y} > \bound{z}$.
        From $\Sign{y}{\bound{x}-1} = \Sign{z}{\bound{x}-1}$, Lemma \ref{th:same-signature} gives $y \bisim_{\bound{x}-1} z$ and, since $\bound{z} < \bound{x}-1$, we get $y \bisim_{\bound{z}} z$.
        Together with $\bound{y} > \bound{z}$, this implies $y \succ z$, which contradicts $z \in \maxRepr{W}$.
        Thus, $\bound{z} \geq \bound{x}-1$, as required.
        From this, $\mathtt{z} \bisim_{\bound{z}} z$, $z \bisim_{\bound{x}-1} y$ and $y \bisim_{\bound{x}-1} \mathtt{y}$, we then get $\mathtt{z} \bisim_{\bound{x}-1} \mathtt{y}$.
        From Lemma \ref{th:same-signature}, this implies $\Sign{\mathtt{z}}{\bound{x}-1} = \Sign{\mathtt{y}}{\bound{x}-1}$.
        Moreover, since $\mathtt{z} \in \maxRepr{\mathtt{W}}$ and $\bound{x} = \bound{\mathtt{x}}$, by Definition \ref{def:canonical-signature} we get $\reprSign{\mathtt{z}} \in \mathtt{S}$.
        Finally, since $z \bisim_{\bound{z}} \mathtt{z}$ and $\bound{z} = \bound{\mathtt{z}}$, Lemma \ref{th:same-signature} gives $\reprSign{z} = \reprSign{\mathtt{z}}$ and, thus, $\reprSign{z} \in \mathtt{S}$, as required.

        \ref{itm:canonical-contr-4}.
        We need to show that for all $x^\star \in W^\star$ there exists $\mathtt{x}^\star \in \mathtt{W}^\star$ such that $L^\star(x^\star) = \mathtt{L}^\star(\mathtt{x}^\star)$, and vice versa.
        We only show the first direction, the other being symmetrical.
        Let $x^\star \in W^\star$.
        By Definition \ref{def:canonical-b-contr}, there exists a world $x \in \maxRepr{W}$ such that $x^\star = \reprSign{x}$.
        From the proof of Item \ref{itm:canonical-contr-2}, there exists $\reprSign{\mathtt{x}} \in \mathtt{W}^\star$ such that $\reprSign{x} = \reprSign{\mathtt{x}}$ and $\bound{x} = \bound{\mathtt{x}}$.
        By repeated applications of Lemma \ref{lem:h-signature}, we get $\Sign{x}{0} = \Sign{\mathtt{x}}{0}$.
        Since Lemma \ref{th:same-signature} gives $x \bisim_0 \mathtt{x}$, by Definition \ref{def:b-bisim}, we get $L(x) = \mathtt{L}(\mathtt{x})$.
        By Definition \ref{def:canonical-b-contr}, we have $L^\star(\reprSign{x}) = L(x)$ and $\mathtt{L}^\star(\reprSign{\mathtt{x}}) = \mathtt{L}(\mathtt{x})$, and we're done.
    \end{proof}
}{Furthermore, canonical $b$-contractions enjoy the following identity property:
    \begin{theorem}\label{th:canonical-contractions}
        $s \bisim_b t$ iff $\canonicalContr{s}{b} = \canonicalContr{t}{b}$.
    \end{theorem}

    \begin{proof}[Proof sketch]
        ($\Leftarrow$): by Theorem~\ref{th:canonical-b-contr}. ($\Rightarrow$): We use Lemma \ref{th:same-signature} and Definitions \ref{def:canonical-signature}--\ref{def:canonical-b-contr} to show that $\canonicalContr{s}{b}$ and $\canonicalContr{t}{b}$ have the same worlds and accessibility relations, and Definition \ref{def:canonical-b-contr} to show that they have the same labels.
    \end{proof}
}

    \begin{example}
        Theorem~\ref{th:canonical-contractions} ensures that $b$-bisimilar states have identical canonical $b$-contractions, unlike rooted contractions where this might not even hold for isomorphic states (cf.\ Example~\ref{ex:non-isomorphic}). We identified the problem to be that the choice of ``unnecessary'' edges to be deleted depended on the total order on $W$, \ie on the naming of worlds. Canonical $b$-contractions solve this problem, as the choice of ``representative edges'' is uniquely determined by the signatures, which are naming independent. 
        In particular, referring again to Example~\ref{ex:non-isomorphic}, since $b(w_3)=1$ in $s$ when $b=3$, the edge $(w_3,w_2)$ of $s$ is  
        replaced by
        $(\reprSign{w_3}, \canonicalSign{w_2}{0})$ in $\canonicalContr{s}{b}$, where $\canonicalSign{w_2}{0} = \min_\lessdot \{ \reprSign{v} \mid v \in \maxRepr{W}, \Sign{v}{0} = \Sign{w_2}{0} \}  =  \min_\lessdot \{ \reprSign{v} \mid v \in \maxRepr{W}, L(v) = L(w_2) = \{ q \} \}$,
        showing that the end node of the edge only depends on the fixed total order on signatures and the label of worlds.     
    \end{example}

    \section{Iterative Bound-Deepening Search}\label{sec:algorithms}
    Unlike more typical search strategies aimed at computing a plan that is as short as possible, our goal is to find 
    a plan with a modal depth as low as possible 
(Definition~\ref{def:action-modal-depth}).
    To this end, we impose a bound $\mathbf{b}$  on the reasoning depth of the planning agent (the modal depth of formulas it can reason about), which in turn allows us to use $\mathbf{b}$-contractions to reduce state sizes.
    To make the intution underlying our algorithm clearer, we first introduce some new helpful notions and corresponding results. 
  \begin{definition}\label{def:exact-approximate}
  A state $s'$ is called an \emph{exact representation} of a state $s$ if $s' \bisim s$. The state $s'$ is called an \emph{approximate representation} (or simply \emph{approximation}) of $s$ if $s' \bisim_b s$ for some $b \geq 0$. In this case, $s'$ is also called a \emph{$b$-approximation} of $s$. If $s'$ is an approximate, but not an exact, representation of $s$, we call it a  \emph{proper approximation}.
  \end{definition}
    \begin{proposition}\label{prop:exact-or-approx} 
    If $s$ is an exact representation of $s'$, then the two states agree on $\Lang_\mathit{dyn}$. If $s$ is a $b$-approximation of $s'$, then they agree on all $\Lang_\mathit{dyn}$-formulas of modal depth at most $b$. In particular,   
    given a state $s$, its canonical $b$-contraction $\canonicalContr{s}{b}$ 
      is a $b$-approximation of $s$ and agrees with it on all $\Lang_\mathit{dyn}$-formulas of modal depth at most $b$. 
\end{proposition}
\begin{proof}
The first statement follows directly from Proposition~\ref{prop:b-bisim}, as any formula of $\Lang_\mathit{dyn}$ is equivalent to a formula in $\Lang$~\cite{book/springer/vanDitmarsch2007}. The second statement follows from Proposition~\ref{prop:b-bisim-dyn}. The third statement follows from the second statement and  Theorem~\ref{th:canonical-b-contr}.
\end{proof}

    To check whether an action sequence $\pi$ is a solution to a planning task $(s_0,\actionSet,\phi_g)$, we need to check whether $s_0 
     \models [\pi]\phi_g$ (Definition~\ref{def:solution}). Proposition~\ref{prop:exact-or-approx} gives us that 
     this is equivalent to checking whether $\canonicalContr{s_0}{\md([\pi]\phi_g)} \models [\pi]\phi_g$. Thus to check whether a particular action sequence $\pi$ is a solution, we can always replace the true initial state $s_0$ by $\canonicalContr{s_0}{\md([\pi]\phi_g)}$, or by any other $\canonicalContr{s_0}{\mathbf{b}}$ with $\mathbf{b} \geq \md([\pi]\phi_g)$. 
     This will be exploited in our algorithm below. The algorithm starts out with an initial value of the bound $\mathbf{b}$ and attempts to compute a solution from the $\mathbf{b}$-contracted initial state $\canonicalContr{s_0}{\mathbf{b}}$. If this fails, $\mathbf{b}$ is incremented, and a new search is performed.

        \subsection{Description of the Planning Algorithm}\label{sec:algorithms-ibds}
            We now describe our planning Algorithm \ref{alg:bounded-search} called Iterative Bound-Deepening Search (\IBDS). The algorithm builds a search tree. Each node of the tree contains a contracted state $\canonicalContr{s}{b}$
            representing some true state $s$. Thus each node state is an approximate representation of the corresponding true state (Definition~\ref{def:exact-approximate} and Proposition~\ref{prop:exact-or-approx}).  
            Formally,  a \emph{node} of the search tree is a pair $n = (s, b)$, where $s$ is the state of $n$ (denoted $n.$state) and $b$ is the \emph{(depth) bound} (denoted $n$.bound). The bound $b$ is intended to guarantee that $n.$state is at least a $b$-approximation of the corresponding true state, and hence agrees with all formulas of the true state up to modal depth $b$ (Proposition~\ref{prop:exact-or-approx}). Later, in Lemma~\ref{lem:path}, we formally prove that this property holds for all nodes in the search tree. For now, we only provide informal arguments. 

            To compute $\canonicalContr{s}{b}$, we use \emph{bounded partition refinements} \cite{journal/logcom/Bolander2022}, a variation of standard partition refinements \cite{journals/siamcomp/Paige1987}.
            The algorithm manipulates via \emph{refinement} operations a partition $P_0$ of $W$ (initially calculated wrt.\ labels):
            for each element $B \in P_0$, called a \emph{block}, and for each relation $R_i$, a refinement produces a new partition $P_1$ by splitting the worlds of $B$ wrt.\ which blocks of $P_0$ they can access via $R_i$.
            Refinements are applied recursively until the sequence $[P_0, \dots, P_b]$ of partitions is computed.
            This gives the bounded bisimulation classes on $W$: the blocks of $P_h$ are the $h$-bisimulation classes of $W$ \cite[Prop.~7]{journal/logcom/Bolander2022}.
            Canonical $b$-contractions are then obtained from the partitions by following Definition \ref{def:canonical-b-contr}.

            \begin{algorithm}[t]
                \caption{Iterative Bound-Deepening Search}\label{alg:bounded-search}
                \begin{algorithmic}[1]
                    \Function{\IBDS}{$(s_0, \actionSet, \phi_g)$}\label{alg:line:ibds}
                        \For{$\mathbf{b} \gets \md(\phi_g)$ \textbf{to} $\infty$}\label{alg:line:it-bs-loop}
                            \State $\pi \gets$ \Call{BoundedSearch}{$(s_0, \actionSet, \phi_g), \mathbf{b}$}
                            \If{$\pi \neq$ \emph{fail}}
                                \Return $\pi$\label{alg:line:it-bs-end-loop}
                            \EndIf
                        \EndFor
                    \EndFunction
                    \smallskip

                    \Function{BoundedSearch}{$(s_0, \actionSet, \phi_g)$, $\mathbf{b}$}
                        \State $\mathit{frontier} \gets \langle$\Call{InitNode}{$s_0$, $\mathbf{b}$}$\rangle$\label{alg:line:frontier-init}
                        \State $\mathit{visited} \gets \varnothing$\label{alg:line:visited-init}

                        \While{$\neg\mathit{frontier}$.\textsc{empty}()}
                            \State $(s, b) \gets \mathit{frontier}$.\textsc{pop}()\label{alg:line:frontier-pop}
                            \State $\mathit{visited}$.\Call{push}{$s$}
                            \If{$s \models \phi_g$} \Return plan to $s$\label{alg:line:return-plan}
                            \EndIf
                            \ForAll{$\alpha \in \actionSet$ such that $b \geq \md(\alpha) + \md(\phi_g)$}\label{alg:line:bs-check}
                                \If{$\alpha$ is applicable in $s$}\label{alg:line:bs-loop}
                                    \State $n' \gets$ \Call{\UpdateNode}{$(s, b)$, $\alpha$}\label{alg:line:bs-update}
                                    \If{$n'$.state $\notin \mathit{visited}$}\label{alg:line:eq-check}
                                        $\mathit{frontier}$.\Call{push}{$n'$}
                                    \EndIf
                                \EndIf
                            \EndFor
                        \EndWhile

                        \State \Return \emph{fail}\label{alg:line:return-fail}
                    \EndFunction
                    \smallskip
                    \Function{InitNode}{$s$, $b$}\label{alg:initnode} 
                        \State \Return $\left(\canonicalContr{s}{b}, b\right)$
                    \EndFunction
                    \smallskip
                    \Function{\UpdateNode}{$(s, b)$, $\alpha$}
                        \State \Return \Call{InitNode}{$s \otimes \alpha$, $b{-}\md(\alpha)$}\label{alg:line:approx-update}
                    \EndFunction
                \end{algorithmic}
            \end{algorithm}

            \begin{algorithm}[t]
                \caption{Improved initialization and child generation}\label{alg:improved-functions}
                \begin{algorithmic}[1]
                    \Function{InitNode}{$s$, $b$, $\mathit{was\_bisim}$}\label{alg:line:init}
                        \State \Return $\left(\canonicalContr{s}{b}, b, (\canonicalContr{s}{b} \bisim s) \land \mathit{was\_bisim}\right)$\label{alg:line:end}
                    \EndFunction
                    \smallskip

                    \Function{\UpdateNode}{$(s, b, \mathit{is\_bisim})$, $\alpha$}\label{alg:line:update}
                        \If{$\mathit{is\_bisim}$}\label{alg:line:update-is-bisim}
                            \Return \Call{InitNode}{$s \otimes \alpha$, $b$, $\mathit{true}$}\label{alg:line:update-exact}
                        \Else\label{alg:line:update-is-not-bisim}
                            \Return \Call{InitNode}{$s \otimes \alpha$, $b{-}\md(\alpha)$, $\mathit{false}$}\label{alg:line:update-approx}
                        \EndIf
                    \EndFunction
                \end{algorithmic}
            \end{algorithm}
            
            In \IBDS\ (line \ref{alg:line:ibds}), the integer $\mathbf{b} \geq 0$ denotes the global maximum modal depth of the formulas that we evaluate.
            Initially, we let $\mathbf{b} = \md(\phi_g)$ (line~\ref{alg:line:it-bs-loop}). We then iteratively call the \textsc{BoundedSearch} algorithm over increasing values of $\mathbf{b}$ until a plan $\pi$ is found (lines~\ref{alg:line:it-bs-loop}-\ref{alg:line:it-bs-end-loop}).  
            \textsc{BoundedSearch} uses breadth-first search (BFS) starting from the initial node $n_0 = \textsc{InitNode}(s_0, \mathbf{b}) = (\canonicalContr{s_0}{\mathbf{b}}, \mathbf{b})$ (line~\ref{alg:line:frontier-init}). 
           In the first call to \textsc{BoundedSearch}, $\mathbf{b}$ has value $\md(\phi_g)$, implying that the state $\canonicalContr{s_0}{\mathbf{b}}$ of the root note $n_0$  at least $\md(\phi_g)$-approximates the true initial state $s_0$. This is sufficient to check whether the goal formula is satisfied in the true initial state $s_0$, since the root state $\canonicalContr{s_0}{\mathbf{b}}$ and $s_0$ will agree on $\phi_g$ (Proposition~\ref{prop:exact-or-approx}). 
           But we are not guaranteed that they will also agree on $[\pi] \phi_g$, for non-empty action sequences $\pi$.


            We keep track of the nodes to be expanded in the $\mathit{frontier}$ queue, 
            and of the states encountered during search in the $\mathit{visited}$ set.
            While $\mathit{frontier}$ contains a node $n = (s, b)$, we extract it from the queue, mark its state as visited and check whether it 
            satisfies the goal formula.
            If it does, we return the plan that led to $s$ (line~\ref{alg:line:return-plan}).
            Otherwise, we process $n$ by 
            generating child nodes $n'$ for all actions $\alpha$ with $b \geq \md(\alpha) + \md(\phi_g)$ that are applicable in $s$ (lines~\ref{alg:line:bs-check}-\ref{alg:line:bs-update}). Each such child node is initialized as $n' = \textsc{\UpdateNode}((s, b), \alpha) = \textsc{InitNode}(s \otimes \alpha, b{-}\md(\alpha)) = (\canonicalContr{s \otimes \alpha}{b {-} \md(\alpha)}, b{-}\md(\alpha))$ (lines~\ref{alg:initnode}-\ref{alg:line:approx-update}). If $n'$.state was not previously visited, we push $n'$ onto $\mathit{frontier}$ and continue the search (line~\ref{alg:line:eq-check}).
            Since $n'$.state is a canonically contracted state, Theorem \ref{th:canonical-contractions} implies that to verify if $n'$.state has already been visited, it suffices to check whether $n'$.state $\in \mathit{visited}$.            Consider now any added child node $n'$. Since we intended for the bound $b$ of a node to guarantee that it at least $b$-approximates the corresponding true state, we need to show that this is preserved when moving from $n$ to $n'$. 
            In other words, we need to guarantee that if $s$ is a $b$-approximation of some true state $t$, then $n'$.state is an $n'$.bound-approximation of $t \otimes \alpha$. If $s$ is a $b$-approximation of $t$, then by Proposition~\ref{prop:b-bisim-product}, $s \otimes \alpha$ is a $(b{-}\md(\alpha))$-approximation of $t \otimes \alpha$. Theorem~\ref{th:canonical-b-contr} then gives that also $n'$.state $= \canonicalContr{s \otimes \alpha}{b{-}\md(\alpha)}$ is a $(b{-}\md(\alpha))$-approximation of $t \otimes \alpha$, and as $n'$.bound $= b{-}\md(\alpha)$, this shows the required. Note furthermore that as $b \geq \md(\alpha) + \md(\phi_g)$, then $n'$.bound $\geq \md(\phi_g)$. Since $n'$.state is $n'$.bound-approximating the true child state $t \otimes \alpha$, this guarantees that $n'$.state agrees with the true child state on the goal formula, and hence that we can correctly verify whether the goal has been achieved after executing $\alpha$ by checking its truth-value in $n'$.state.

            Finally, if the frontier is empty and no plan was found,
            we return \emph{fail} and proceed to the next iteration of \IBDS.

            As seen, the \textsc{\IBDS} strategy attempts to compute plans with the lowest possible modal depth (lowest value of $\mathbf{b}$).    
            Moreover, by keeping track of node bounds we can reduce the size of visited states via bounded contractions.
            As we now show, this reduction can be further improved by finding lower bounds for contractions.

        \subsection{Improving Bounds for Contractions}\label{sec:algorithms-improvements}
            \begin{figure*}[t]
                \centering
                \input{img/ex-ibds/ex-ibds.tex}
                \caption{
                    Partial execution of \textsc{mixed-ibds} on the planning task $T$ of Example~\ref{ex:ibds}.
                    Gray worlds and edges denote the parts of a state that are removed due to bounded contractions (\eg the rightmost world of $s_0$ is not included in $\canonicalContr{s_0}{2}$).
                    For brevity, each world is labelled by a pair $(x, y)$, meaning that in that world agents $a$ and $b$ have numbers $x$ and $y$, respectively.
                    Dashed borders denote nodes that can not be expanded with any action, and double borders denote nodes that satisfy the goal.
                }
                \label{fig:ibds-execution}
            \end{figure*}
            We now show how to achieve tighter bounds for contractions, further improving the algorithm. As explained above, a state $s$ of the search tree is guaranteed to $b$-approximate the corresponding true state $t$.   
            Sometimes, however, it might be that $s$ is even an exact representation of $t$, \ie bisimilar to it.
            In other words, bounded contractions do not always yield proper approximations (Definition~\ref{def:exact-approximate}).
            We use this key observation to improve our algorithm as follows.
            We now let a node be a \emph{triple} $n = (s, b, \mathit{is\_bisim})$, where the $b$-contracted state $s = \canonicalContr{t}{b}$ ($t$ being the true state) and $b$ are as above, and $\mathit{is\_bisim}$ (denoted $n$.is\_bisim) is a boolean representing whether $\canonicalContr{t}{b} \bisim t$ holds, \ie whether the state of a node is an exact representation of the true state ($\mathit{is\_bisim}$ true) or not ($\mathit{is\_bisim}$ false).
            We exploit this new information to improve functions \textsc{InitNode} and \textsc{\UpdateNode} (Algorithm~\ref{alg:improved-functions}).
            \textsc{InitNode} now takes an extra boolean parameter $\mathit{was\_bisim}$ being true if all states of the antecedents of the node being initialized are exact representations, and it initializes a new node $n'$ together with its bisimilarity information (line~\ref{alg:line:end}).
            We then initialize the frontier with $\langle$\Call{InitNode}{$s_0$, $\mathbf{b}, \mathit{true}$}$\rangle$ in line \ref{alg:line:frontier-init} of \textsc{BoundedSearch}.
            \textsc{\UpdateNode} now generates a child node $n'$, from a node $n = (s, b, \mathit{is\_bisim})$ and an action $\alpha$, as follows. If $\mathit{is\_bisim}$ is true (line \ref{alg:line:update-is-bisim}),
            then the node state $s$ is bisimilar to the true state $t$, and hence also $s \otimes \alpha \bisim t \otimes \alpha$, by Proposition~\ref{prop:bisim-product}. Thus also the state of the child node is an exact representation of the corresponding true state, and we can hence preserve the current bound by letting $n'$.bound $= b$. If $\mathit{is\_bisim}$ is false (line \ref{alg:line:update-is-not-bisim}), then we only know that $s$ is a $b$-approximation of the true state and we proceed as in Algorithm~\ref{alg:bounded-search}.

            We can think of the new algorithm as having two different modes for generating child nodes:
            An \emph{exact mode} -- invoked if in the path leading to the new node all states are exact representations, and an \emph{approximate mode} -- when some of the states are approximations.
            The algorithm starts in exact mode and, while it runs in this mode, the bound is preserved, as we are guaranteed that no information is lost due to bounded contractions.
            As soon as we generate an approximate state, we enter into the approximate mode and continue the search as in Algorithm \ref{alg:bounded-search}.
            To avoid ambiguity, from now on we refer to Algorithm~\ref{alg:bounded-search} as \textsc{approx-ibds} and to the improved version as \textsc{mixed-ibds}.

            \textsc{mixed-ibds} improves \textsc{approx-ibds} in several ways.
            First, the number of iterations in \textsc{mixed-ibds} is reduced in general.
            Consider a path $n_0 \xrightarrow{\alpha_1} \cdots \xrightarrow{\alpha_l} n_l$ visited by \textsc{BoundedSearch}($T, \mathbf{b}$), and let $n_k = (s_k, b_k, \mu_k)$ for each $0 \leq k \leq l$.
            In \textsc{approx-ibds}, we have $b_k = \mathbf{b} - \sum_{h \leq k} \md(\alpha_h)$, as we subtract the modal depth of each applied action from the initial bound (line~\ref{alg:line:approx-update}).
            In \textsc{mixed-ibds} this does not necessarily happen for each action, so
            \textsc{mixed-ibds} will generally be able to expand more nodes and add more children without increasing the bound $\mathbf{b}$, hence generally require fewer iterations of \textsc{BoundedSearch}.   
            %
            Second, we improve bounds for contractions: As less iterations are needed, the global bound $\mathbf{b}$ is smaller, so we can take $b$-contractions with lower values of $b$, further reducing the size of states.
            Finally, we implemented an important optimization in \textsc{mixed-ibds} (omitted in the pseudocode for readability):
            We preserve all nodes $n$ of the search tree with $n.\text{is\_bisim} = \mathit{true}$ across different iterations of \textsc{mixed-ibds}, as in each iteration they would simply be recomputed as in the previous ones.
            %
            As we show in our experiments, this optimization avoids redundant computation and provides effective improvements wrt.\ \textsc{approx-ibds}.
            



            \begin{example}\label{ex:ibds}
                We now show the execution of \textnormal{\textsc{mixed-ibds}} on the planning task $T = (s_0, \{ \mathit{ann}_{a,b}, \mathit{ann}_{b,a} \}, \phi_g)$, where
                $s_0$ is from Figure~\ref{fig:state}, $\mathit{ann}_{i,j} = \mathit{ann}(\bigwedge_{0 \leq k \leq N} \neg \B{i} \mathit{has}(j, k))$ (Example~\ref{ex:action}) is the public announcement of the fact that agent $i$ does not know $j$'s number, and $\phi_g = \B{b} \B{a} \mathit{has}(b, 4)$.
                The algorithm invokes \textnormal{\textsc{BoundedSearch}} with increasing reasoning bounds $\mathbf{b}$, starting from $\mathbf{b} = \md(\phi_g) = 2$.
                
                \emph{$\mathbf{b} = 2$:}
                The algorithm performs a BFS starting from the initial node $n_0 = \textnormal{\textsc{InitNode}}(s_0, 2, \mathit{true})$.
                \textnormal{\textsc{InitNode}} first computes the canonical $2$-contraction of $s_0$, represented by $n_0$ in Figure~\ref{fig:ibds-execution}. 
                By Proposition~\ref{prop:bisim}, $s_0 \not\bisim \canonicalContr{s_0}{2}$ since they disagree on $\D{a}\D{b}\D{a}\mathit{has}(b,0)$, so we get $n_0 = (\canonicalContr{s_0}{2}, 2, \mathit{false})$.
                After extracting $n_0$ from the frontier (line~\ref{alg:line:frontier-pop}), since $\canonicalContr{s_0}{2} \not\models \phi_g$ and both actions have modal depth $1$, the condition $\mathbf{b} \geq \md(\alpha) + \md(\phi_g)$ (line~\ref{alg:line:bs-check}) fails, so no children are generated and the search continues.

                \emph{$\mathbf{b} = 3$:}
                With the same reasoning as above, we get $n'_0 = (\canonicalContr{s_0}{3}, 3, \mathit{false})$.
                Since $\canonicalContr{s_0}{3} \not\models \phi_g$ and both actions are applicable in $\canonicalContr{s_0}{3}$ satisfying $\mathbf{b} \geq \md(\alpha) + \md(\phi_g)$, we generate $n_1 = (\canonicalContr{s_1}{2}, 2, \mathit{false})$ and $n_1' = (\canonicalContr{s_1'}{2}, 2, \mathit{false})$, where $s_1 = \canonicalContr{s_0}{3} \otimes \mathit{ann}_{a,b}$ and $s_1' = \canonicalContr{s_0}{3} \otimes \mathit{ann}_{b,a}$.
                As shown in Example~\ref{ex:action}, public announcements delete worlds that do not satisfy the announced formula.
                Thus, $\mathit{ann}_{a,b}$ removes the leftmost world from $\canonicalContr{s_0}{3}$ (as in that world agent $a$ knows that $b$ has number $4$), and $\mathit{ann}_{b,a}$ deletes the rightmost one.
                Note that $\canonicalContr{s_1'}{2} = \canonicalContr{s_0}{2}$, so $n_0$ and $n'_1$ are identical; in Figure~\ref{fig:ibds-execution} we merge them for clarity, despite being separate nodes.
                Both $n_1$ and $n_1'$ have bound $2$, which is insufficient to proceed (as previously shown), so the search advances to the next iteration.

                \emph{$\mathbf{b} = 4$:}
                Due to space constraints, we only describe the path of the search tree induced by the action sequence $\pi = \mathit{ann}_{b,a}, \mathit{ann}_{a,b}, \mathit{ann}_{b,a}$ (shown in Figure~\ref{fig:ibds-execution}).
                Since $\canonicalContr{s_0}{4}$ (Figure~\ref{fig:ibds-execution}) and $s_0$ are isomorphic, they are also bisimilar, so the initial node of the path is $n''_0 = (\canonicalContr{s_0}{4}, 4, \mathit{true})$.
                Following the computation of the path, we can show that the reasoning bound is sufficiently high to visit all nodes (updates are computed as in the previous iteration).
                It can also be shown that no other previously visited path leads to a goal state.
                Thus, node $n''_3$ is going to be eventually visited, and since $n''_3$.state $\models \phi_g$ the action sequence $\pi$ is returned.
            \end{example}

\section{Soundness, Completeness and Complexity}
In the following, $T = (s_0,\actionSet,\phi_g)$ is a planning task and $\mathbf{b} \geq \md(\phi_g)$ a constant.
For a sequence $\pi = \alpha_1, \dots, \alpha_l$ of actions, $\graphPath{\pi}$ denotes a path
$n_0 \xrightarrow{\alpha_1} \cdots \xrightarrow{\alpha_l} n_l$
in the graph visited by \textsc{BoundedSearch}($T, \mathbf{b}$).
\begin{lemma}\label{lem:path}
    Let $n_k = (t_k, b_k, \mu_k)$ be the last node of $\graphPath{\pi}$, where $\pi = \alpha_1, \dots, \alpha_k$.
    Then:
    \begin{compactenum}[(1)]
        \item\label{itm:bisim_true} If $\mu_k = \mathit{true}$, then $t_k \bisim s_0 \otimes \pi$;
        \item\label{itm:bisim_false} If $\mu_k = \mathit{false}$, then $t_k \bisim_{b_k} s_0 \otimes \pi$;
        \item\label{itm:inequalities} $b_k \geq \mathbf{b} - \Sigma_{i \leq k} \md(\alpha_i)$ and $b_k \geq \md(\phi_g)$.
    \end{compactenum}
\end{lemma}

\ifthenelse{\boolean{expandedVersion}}{
\begin{proof}
    By induction on $k$.
    Base case ($k=0$): We have $n_0 = \textsc{InitNode}(s_0,\mathbf{b},true)= (t_0, \mathbf{b}, \mu_0)$, $t_0 = \canonicalContr{s_0}{\mathbf{b}}$, $\mu_k = (t_0 \bisim s_0)$.
    Thus $t_0 \bisim_{\mathbf{b}} s_0$ (Thm.~\ref{th:canonical-b-contr}), proving (\ref{itm:bisim_false}).
    By def.\ of $\mu_0$, we get (\ref{itm:bisim_true}).
    Since we assumed $\mathbf{b} \geq \md(\phi_g)$, we get (\ref{itm:inequalities}).
    Induction step ($k > 0$):
    Let $n_k = \textsc{ChildNode}(n_{k-1},\alpha_k)$ and $n_{k-1} = (t_{k-1}, b_{k-1}, \mu_{k-1})$.
    Then:
    
    \emph{(i)} If $\mu_{k-1} = \mathit{true}$:
    Letting $s_k = t_{k-1} \otimes \alpha_k$ and $b_k = b_{k-1}$, we get $n_k = \textsc{InitNode}(s_k, b_k, \mathit{true}) = (t_k, b_k, \mu_k)$ where $t_k = \canonicalContr{s_k}{b_k}$ and $\mu_k = (t_k \bisim s_k)$ (Algorithm~\ref{alg:improved-functions}, line \ref{alg:line:update-exact}).
    Since $\mu_{k-1} = \mathit{true}$, i.h.\ gives $t_{k-1} \bisim s_0 \otimes \pi_{\leq k-1}$,
    which by Proposition~\ref{prop:bisim-product} implies $s_k = t_{k-1} \otimes \alpha_k \bisim s_0 \otimes \pi_{\leq k}$.
    If $\mu_k = \mathit{true}$ (\ie $t_k \bisim s_k$) we get $t_k \bisim s_k \bisim s_0 \otimes \pi_{\leq k}$, proving (\ref{itm:bisim_true}).
    Otherwise, since $t_k = \canonicalContr{s_k}{b_k}$ we get $t_k \bisim_{b_k} s_k \bisim s_0 \otimes \pi_{\leq k}$, proving (\ref{itm:bisim_false}).
    As $b_k = b_{k-1}$, i.h.\ immediately gives us (\ref{itm:inequalities}).

    \emph{(ii)} If $\mu_{k-1} = \mathit{false}$:
    Letting $s_k = t_{k-1} \otimes \alpha_k$ and $b_k = b_{k-1} {-} \md(\alpha_k)$, we get $n_k = \textsc{InitNode}(s_k, b_k, \mathit{false}) = (t_k, b_k, \mathit{false})$, where $t_k = \canonicalContr{s_k}{b_k}$ (Algorithm~\ref{alg:improved-functions}, line \ref{alg:line:update-approx}).
    Since $\mu_{k-1} = \mathit{false}$, (\ref{itm:bisim_true}) trivially holds.
    Moreover, i.h.\ gives $t_{k-1} \bisim_{b_{k-1}} s_0 \otimes \pi_{\leq k-1}$.
    Since $b_k = b_{k-1} {-} \md(\alpha_k)$, Proposition~\ref{prop:b-bisim-product} then gives $s_k = t_{k-1} \otimes \alpha_k \bisim_{b_k} s_0 \otimes \pi_{\leq k}$,
    and hence, since $t_k = \canonicalContr{s_k}{b_k}$, $t_k \bisim_{b_k} s_k \bisim_{b_k} s_0 \otimes \pi_{\leq k}$, proving (\ref{itm:bisim_false}).
    Since $n_k$ was generated from $n_{k-1}$ in line \ref{alg:line:bs-update} of the algorithm, $n_{k-1}$ satisfied the condition in line \ref{alg:line:bs-check}, so $b_{k-1} \geq md(\alpha_k) + md(\phi_g)$.
    Hence we get $b_k = b_{k-1} - md(\alpha_k) \geq md(\phi_g)$.
    As i.h.\ gives us $b_{k-1} \geq b_0 - \Sigma_{i<k} \md(\alpha_i)$, we get $b_k = b_{k-1} {-} \md(\alpha_k) \geq b_0 - \Sigma_{i \leq k} \md(\alpha_i)$, proving (\ref{itm:inequalities}).
\end{proof}
}{
\begin{proof}[Proof sketch]
    By induction on $k$.
    Base case ($k = 0$):
    Immediate by construction, Proposition~\ref{prop:exact-or-approx}, and $\mathbf{b} \geq \md(\phi_g)$.
    Induction ($k > 0$):
    Assume the properties for $n_{k-1}$.
    If $\mu_{k-1}$ is $\mathit{true}$, Proposition~\ref{prop:bisim-product} gives (\ref{itm:bisim_true}-\ref{itm:bisim_false}), and the bound is unchanged.
    Otherwise, (\ref{itm:bisim_true}) is trivial, (\ref{itm:bisim_false}) follows by Proposition~\ref{prop:b-bisim-product}, and the bound decreases by $\md(\alpha_k)$, preserving (\ref{itm:inequalities}).
\end{proof}
}

\begin{theorem}[Soundness]\label{th:soundness}
    If \textnormal{\textsc{BoundedSearch}}$(T, \mathbf{b})$ returns an action sequence $\pi$, then $\pi$ is a solution to $T$.
\end{theorem} 
\begin{proof}
    
        



    Let $n_k = (t_k, b_k, \mu_k)$ be the last node of $\graphPath{\pi}$.
    From line \ref{alg:line:return-plan} of the algorithm, we get $t_k \models \phi_g$.
    Since $t_k \bisim_{b_k} s_0 \otimes \pi$ and $b_k \geq \md(\phi_g)$ (by Lemma \ref{lem:path}), Proposition~\ref{prop:b-bisim} gives us that $t_k$ and $s_0 \otimes \pi$ agree on formulas up to depth $\md(\phi_g)$, and hence $s_0 \otimes \pi \models \phi_g$. Thus, $ \pi$ is a solution to $T$.
   \end{proof}

   We now present the parameters used for our completeness and complexity results, originally introduced in the context of epistemic plan verification~\cite{vandepol2018parameterized,journal/logcom/Bolander2022}.
   We let $\mathsf{a} = |\agentSet|$, $\mathsf{p} = |\atomSet|$,
   $\mathsf{c} = \max \{ \md(\alpha) \mid \alpha \in \actionSet \}$, $\mathsf{o} = \md(\phi_g)$, and 
   $\mathsf{u}$ be the maximal allowed solution length. We use $|T|$ for the size of $T$, \ie the sum of the sizes of $s_0$, $\actionSet$ and $\phi_g$.
\begin{theorem}[Completeness]\label{th:completeness}
   If $T$ has a solution of length $\mathsf{u}$, then  \textnormal{\textsc{BoundedSearch}}$(T,\mathsf{c} \mathsf{u} + \mathsf{o})$ will find a solution to it.
\end{theorem}

\ifthenelse{\boolean{expandedVersion}}{
\begin{proof}
    Suppose $T$ has a solution $\pi = \alpha_1,\dots,\alpha_l$ with $l \leq \mathsf{u}$.
    We first show by induction on $k$ that the search graph computed by the algorithm include $\graphPath{\pi_{\leq k}}$ for all $k \leq l$.
    The base case is trivial.
    Induction step: Suppose the algorithm has visited $\graphPath{\pi_{\leq k}}$, with last node $n_k = (t_k, b_k, \mu_k)$.
    We now show that the algorithm will eventually expand $n_k$ by applying $\alpha_{k+1}$.
    By Lemma \ref{lem:path} we get $t_k \bisim_{b_k} s_0\otimes \pi_{\leq k}$, $b_k \geq \mathbf{b} - \Sigma_{i \leq k} \md(\alpha)$, and $b_k \geq \md(\phi_g)$, where $\mathbf{b} = \mathsf{c} \mathsf{u} + \mathsf{o}$.
    Since $\md(\alpha) \leq \mathsf{c}$ for all $\alpha \in \actionSet$ and $k \leq l-1 \leq \mathsf{u}-1$, we get $b_k \geq \mathbf{b} - \Sigma_{i \leq k} \md(\alpha) \geq \mathsf{c} \mathsf{u} + \mathsf{o} - k\mathsf{c} \geq  \mathsf{c} \mathsf{u} + \mathsf{o} - \mathsf{c}(\mathsf{u}-1) = \mathsf{c} + \mathsf{o} \geq \md(\alpha_{k+1}) + \md(\phi_g)$, and thus the test in line \ref{alg:line:bs-check} will be satisfied.
    Moreover, since $\alpha_{k+1}$ is applicable in $s_0 \otimes \pi_{\leq k}$ ($\pi$ is a solution) and $t_k \bisim_{b_k} s_0 \otimes \pi_{\leq k}$, $\alpha_{k+1}$ is applicable in $t_k$ too, so $\alpha_{k+1}$ will eventually be chosen in line~\ref{alg:line:bs-loop}, causing $\textsc{ChildNode}(n_k,\alpha_{k+1}, \mathsf{c} \mathsf{u} + \mathsf{o})$ to be called.

    Thus the algorithm will eventually compute $\graphPath{\pi}$.
    It only remains to show that when the last node $n_l = (t_l, b_l, \mu_l)$ of the path is popped, a solution will be returned, \ie that $t_l \models \phi_g$.
    From Lemma \ref{lem:path}, we get $t_l \bisim_{b_l} s_0 \otimes \pi$ and $b_l \geq \md(\phi_g)$
    and from Proposition~\ref{prop:b-bisim} that $t_l$ and $s_0 \otimes \pi$ agree on all formulas up to depth $\md(\phi_g)$.
    Since $\pi$ is a solution to $T$, we get $s_0 \otimes \pi \models \phi_g$ and hence $t_l \models \phi_g$, as required. 
\end{proof}
}{
\begin{proof}[Proof sketch]
    Assume $T$ has a solution $\pi$ of length $\mathsf{u}$. Lemma~\ref{lem:path} is then used to prove that   
    the search tree constructed by 
    \textnormal{\textsc{BoundedSearch}}$(T,\mathbf{b})$
    will contain $\graphPath{\pi}$ when   $\mathbf{b} \geq \mathsf{c} \mathsf{u} + \mathsf{o}$. 
    When the last node of the path is reached, Lemma~\ref{lem:path} and Proposition~\ref{prop:b-bisim} ensure that the goal formula holds, so the algorithm finds a solution.
\end{proof}
}

Theorem~\ref{th:completeness}
shows that \textsc{BoundedSearch} will always solve a planning task $T$ if the bound is high enough. 
Since $\IBDS$ calls $\textsc{BoundedSearch}$ iteratively with increasing bounds, it will eventually find a solution if one exists. 

\begin{theorem}[Complexity]\label{th:complexity}
    \textnormal{$\textsc{BoundedSearch}(T,b)$} runs in time $|T|^{O(1)} \exp_2^{b+1} O(\mathsf{a} {+} \mathsf{p})$.\footnote{Where the iterated exponential $\exp^n_a x$ is defined as follows: $\exp^0_a x = x$ and $\exp^{n+1}_a x = a^{(\exp^n_a x)}$.}
\end{theorem}

\ifthenelse{\boolean{expandedVersion}}{
\begin{proof}
    In $\textsc{BoundedSearch}(T,b)$ the bound of the initial node is $b$ (Algorithm~\ref{alg:bounded-search}, line~\ref{alg:line:frontier-init}), and the bound of each visited child node is either preserved from the parent node (Algorithm~\ref{alg:improved-functions}, line~\ref{alg:line:update-exact}), or it is lower than the parent's bound (Algorithm~\ref{alg:improved-functions}, line~\ref{alg:line:update-approx}), so the bound of any visited node is at most $b$.
    Furthermore, the condition in line~\ref{alg:line:bs-check} guarantees that the bound of any visited node is at least $0$.
    Let then $0 \leq h \leq b$ be a bound value, and let $\bm{\sigma}_h$ be the set of $h$-signatures that can be built from the agents and atoms in $T$. By induction on $h$, we now prove $| \bm{\sigma}_h | = \exp_2^{h+1} O(\mathsf{a} {+} \mathsf{p})$. Base case: As $\sigma_0(w) = (L(w),\varnothing)$, we have $| \bm{\sigma}_0 | = 2^\mathsf{p} = 2^{O(\mathsf{a} {+} \mathsf{p})} = \exp_2^1 O(\mathsf{a} {+} \mathsf{p})$. Induction step: $\sigma_{h+1}(w) = (L(w),\Sigma_{h+1}(w))$, where $\Sigma_{h+1}(w)$ is a mapping from $\agentSet$ into $2^{\bm{\sigma}_h}$.
    There are $\mathsf{a}$ agents and $| \bm{\sigma}_{h} |$ $h$-signatures, so there are $(2^{| \bm{\sigma}_{h} |})^\mathsf{a}$ such mappings.
    Thus using i.h. we have $| \bm{\sigma}_{h+1} | = 2^\mathsf{p} (2^{|\bm{\sigma}_{h}|})^ \mathsf{a}  = 2^{\mathsf{p} + \mathsf{a}|\bm{\sigma}_{h}| } = 2^{\mathsf{p} + \mathsf{a}(\exp_2^{h+1} O(\mathsf{a} {+} \mathsf{p}))} = 2^{\exp_2^{h+1} O(\mathsf{a} {+} \mathsf{p})} = \exp_2^{h+2} O(\mathsf{a} {+} \mathsf{p})$, as required.
    Let now $s = (M,w_s)$ and $t = (N,w_t)$ be two states and let $0 \leq h \leq b$. 
    Since $\canonicalContr{s}{h} = \canonicalContr{t}{h}$ iff $s \bisim_h t$ iff $\Sign{w_s}{h} = \Sign{w_t}{h}$ (Theorem~\ref{th:canonical-contractions} and Lemma~\ref{th:same-signature}), the number of $h$-contracted states equals the number of $h$-signatures.
    As we are doing graph search and the bounds of visited nodes range from $0$ to $b$, the search graph contains at most one node per (canonically) $h$-contracted state \ie at most $|\bm{\sigma}_b| = \exp_2^{b+1} O(\mathsf{a} {+} \mathsf{p}) + \exp_2^{b} O(\mathsf{a} {+} \mathsf{p}) + \dots + \exp_2^{1} O(\mathsf{a} {+} \mathsf{p}) = \exp_2^{b+1} O(\mathsf{a} {+} \mathsf{p})$ nodes.

    The proof of Proposition 9 in \citeauthor{journal/logcom/Bolander2022}~[\citeyear{journal/logcom/Bolander2022}] shows that the number of worlds in each $b$-contracted state is bounded by $g(b,\mathsf{a},\mathsf{p})$ where $g(0,\mathsf{a},\mathsf{p}) = 2^\mathsf{p}$ and $g(h{+}1,\mathsf{a},\mathsf{p}) = 2^{\mathsf{a} g(h,\mathsf{a},\mathsf{p})}g(h,\mathsf{a},\mathsf{p})$. By an induction proof almost identical to the above, we easily get that $g(b,\mathsf{a},\mathsf{p}) = \exp_2^{b+1} O(\mathsf{a} {+} \mathsf{p})$. Hence also the size of each $b$-contracted state is bounded by $\exp_2^{b+1} O(\mathsf{a} {+} \mathsf{p})$.    

    For each node $n$ of the search graph, the algorithm computes at most $|\actionSet|$ $b$-contracted updates (product updates followed by $b$-contractions). From the proof of Theorem 4 in \citeauthor{journal/logcom/Bolander2022}~[\citeyear{journal/logcom/Bolander2022}], we know that the computation time of each $b$-contracted update is polynomial in the size of $n$.state and $|T|$. As each state is $b$-contracted, $n$.state has size $\exp_2^{b+1} O(\mathsf{a} {+} \mathsf{p})$, so each $b$-contracted update takes time $|T|^{O(1)} \exp_2^{b+1} O(\mathsf{a} {+} \mathsf{p})$. As we are performing at most $|\actionSet| = O(|T|)$
    of these operations per state, and there are at most $\exp_2^{b+1} O(\mathsf{a} {+} \mathsf{p})$ states, we get that the total computation time for all the $b$-contracted updates (including checking for the applicability of actions) is still 
    $|T|^{O(1)}\exp_2^{b+1} O(\mathsf{a} {+} \mathsf{p})$.
    We compare each $b$-contracted update to all existing states in the search graph. The total computation cost of this is polynomial in the size and number of states, so also bounded by $|T|^{O(1)} \exp_2^{b+1} O(\mathsf{a} {+} \mathsf{p})$. 
    Finally, for each node $n$, we check whether $n\text{.state} \models \phi_g$. Model checking is polynomial in the size of the state and the formula, so also bounded by $|T|^{O(1)} \exp_2^{b+1} O(\mathsf{a} {+} \mathsf{p})$.
\end{proof}
}{
\begin{proof}[Proof sketch]
    The algorithm explores a search tree where each node corresponds to a canonically $b$-contracted state.
    By induction on $b$, we show that the number and size of such states are both bounded by $\exp_2^{b+1} O(\mathsf{a} {+} \mathsf{p})$.
    For each node, at most $O(|T|)$ updates, $b$-contractions and formulas checks are performed, each taking time polynomial in $|T|$ and the size of the state.
\end{proof}
}

The theorem shows that computing epistemic plans is tractable when we put a fixed bound on the reasoning depth and the numbers of agents and atoms ($\mathbf{b}$, $\mathsf{a}$ and $\mathsf{p}$). This might be a quite realistic assumptions for many practical applications.  
The theorem generalises several existing results. If $\mathsf{c}=0$, then $\textsc{BoundedSearch}(T,\mathsf{o})$ will find a solution to $T$ (Theorem~\ref{th:completeness}), and it will do so in time $|T|^{O(1)} \exp_2^{\mathsf{o}+1} O(\mathsf{a} {+} \mathsf{p})$ (Theorem~\ref{th:complexity}).
This shows that we can solve propositional planning tasks (\ie tasks where all pre- and postconditions are propositional) in $(\md(\phi_g){+}1)$-\textsc{ExpTime}, hence providing a new proof of an existing  result~\cite{conf/ijcai/Yu2013,phd/hal/Maubert2014}. Also, if $\mathbf{b} = \mathsf{c} \mathsf{u} + \mathsf{o}$, then 
 $\textsc{BoundedSearch}(T,\mathbf{b})$ is complete (Theorem~\ref{th:completeness}), and it runs in time $|T|^{O(1)} \exp_2^{\mathbf{b} +1 } O(\mathsf{a} {+} \mathsf{p})$ (Theorem~\ref{th:complexity}). This proves that the plan existence problem in epistemic planning is fixed-parameter tractable (FPT) with parameters $\mathsf{a}$, $\mathsf{c}$, $\mathsf{o}$, $\mathsf{p}$ and $\mathsf{u}$, generalizing the existing FPT result for epistemic plan verification~\cite{journal/logcom/Bolander2022}. As the aforementioned paper shows that plan verification is \emph{not} FPT for any subset of these parameters, 
 this is the strongest FPT result for the plan existence problem that one can get (with the given parameter set). 
 


    \pgfplotstableread{time_results/all_results.csv}\allResults
\pgfplotstableread{time_results/efp_results.csv}\efpResults
\begin{figure*}[t]
    \centering
    \subfloat{\label{fig:results-exact}\scalebox{0.8075}{\begin{tikzpicture}
    \centering
    \begin{axis}[
        ymode=log,xmode=log,
        grid=major,
        legend entries={\tiny\textsc{ActMudChild},\tiny\textsc{CoinBox},\tiny\textsc{CollabCom},\tiny\textsc{ConsecNum},\tiny\textsc{Eavesdropping},\tiny\textsc{Gossip},\tiny\textsc{Grapevine},\tiny\textsc{SelectCom},\tiny\textsc{Tiger}},
        legend style={at={(0.31,0.97), anchor=north east}},
        legend cell align={left},
        xlabel={\small $x$: \textsc{exact-ibds} -- $y$: \textsc{mixed-ibds}},
        axis x line*=bottom,
        axis y line*=left,
        xmin=0.1,
        xmax=2000,
        ymin=0.1,
        ymax=2000,
        xticklabels={,, $1^{\color{white} 1}$, $10^{\color{white} 1}$, $10^2$, $10^3$},
        yticklabels={,, $1^{\color{white} 1}$, $10^{\color{white} 1}$, $10^2$, $10^3$},
        tick label style={font=\tiny}
        ]
        \addplot[
            scatter,
            only marks,
            point meta=explicit symbolic,
            scatter/classes={
                acm ={palette4, mark=*,               mark size=2pt},
                cb  ={palette2, mark=+,               mark size=2.5pt},
                cc  ={palette3, mark=x,               mark size=2.5pt},
                cn  ={palette1, mark=star,            mark size=2.5pt},
                eav ={palette9, mark=asterisk,        mark size=2.5pt},
                gos ={palette5, mark=10-pointed star, mark size=2.5pt},
                gra ={palette6, mark=square*,         mark size=2pt},
                sc  ={palette7, mark=triangle*,       mark size=2pt},
                tig ={palette8, mark=diamond*,        mark size=2pt}
            }
        ] table [
            x    index={10},
            y    index={15},
            meta index={0} ] {\allResults};
        \draw[dashed,thick,darkgray] (0.001,0.001) -- (2000,2000);
        \draw[densely dotted,red] (1200,0.001) -- (1200,2000);
        \node[red, rotate=90] at (1700,12) {\tiny \begin{tabular}{c} \textsc{exact-ibds} timeout (104 instances) \end{tabular}};
        \draw[densely dotted,red] (0.001,1200) -- (2000,1200);
        \node[red, rotate=0 ] at (22,1500) {\tiny \begin{tabular}{c} \textsc{mixed-ibds} timeout (44 instances) \end{tabular}};
    \end{axis}
\end{tikzpicture}}}
    \hfill
    \subfloat{\label{fig:results-approx}\scalebox{0.8075}{\begin{tikzpicture}
    \centering
    \begin{axis}[
        ymode=log,xmode=log,
        grid=major,
        legend entries={\tiny\textsc{ActMudChild},\tiny\textsc{CoinBox},\tiny\textsc{CollabCom},\tiny\textsc{ConsecNum},\tiny\textsc{Eavesdropping},\tiny\textsc{Gossip},\tiny\textsc{Grapevine},\tiny\textsc{SelectCom},\tiny\textsc{Tiger}},
        legend style={at={(0.31,0.97), anchor=north east}},
        legend cell align={left},
        xlabel={\small $x$: \textsc{approx-ibds} -- $y$: \textsc{mixed-ibds}},
        axis x line*=bottom,
        axis y line*=left,
        xmin=0.1,
        xmax=2000,
        ymin=0.1,
        ymax=2000,
        xticklabels={,, $1^{\color{white} 1}$, $10^{\color{white} 1}$, $10^2$, $10^3$},
        yticklabels={},
        tick label style={font=\tiny}
        ]
        \addplot[
            scatter,
            only marks,
            point meta=explicit symbolic,
            scatter/classes={
                acm ={palette4, mark=*,               mark size=2pt},
                cb  ={palette2, mark=+,               mark size=2.5pt},
                cc  ={palette3, mark=x,               mark size=2.5pt},
                cn  ={palette1, mark=star,            mark size=2.5pt},
                eav ={palette9, mark=asterisk,        mark size=2.5pt},
                gos ={palette5, mark=10-pointed star, mark size=2.5pt},
                gra ={palette6, mark=square*,         mark size=2pt},
                sc  ={palette7, mark=triangle*,       mark size=2pt},
                tig ={palette8, mark=diamond*,        mark size=2pt}
            }
        ] table [
            x    index={18},
            y    index={15},
            meta index={0} ] {\allResults};
        \draw[dashed,thick,darkgray] (0.001,0.001) -- (2000,2000);
        \draw[densely dotted,red] (1200,0.001) -- (1200,2000);
        \node[red, rotate=90] at (1700,12) {\tiny \begin{tabular}{c} \textsc{approx-ibds} timeout (52 instances) \end{tabular}};
        \draw[densely dotted,red] (0.001,1200) -- (2000,1200);
        \node[red, rotate=0 ] at (22,1500) {\tiny \begin{tabular}{c} \textsc{mixed-ibds} timeout (44 instances) \end{tabular}};
    \end{axis}
\end{tikzpicture}}}
    \hfill
    \subfloat{\label{fig:results-rooted}\scalebox{0.8075}{\begin{tikzpicture}
    \centering
    \begin{axis}[
        ymode=log,xmode=log,
        grid=major,
        legend entries={\tiny\textsc{ActMudChild},\tiny\textsc{CoinBox},\tiny\textsc{CollabCom},\tiny\textsc{ConsecNum},\tiny\textsc{Eavesdropping},\tiny\textsc{Gossip},\tiny\textsc{Grapevine},\tiny\textsc{SelectCom},\tiny\textsc{Tiger}},
        legend style={at={(0.31,0.97), anchor=north east}},
        legend cell align={left},
        xlabel={\small $x$: \textsc{rooted-ibds} -- $y$: \textsc{mixed-ibds}},
        axis x line*=bottom,
        axis y line*=left,
        xmin=0.1,
        xmax=2000,
        ymin=0.1,
        ymax=2000,
        xticklabels={,, $1^{\color{white} 1}$, $10^{\color{white} 1}$, $10^2$, $10^3$},
        yticklabels={},
        tick label style={font=\tiny}
        ]
        \addplot[
            scatter,
            only marks,
            point meta=explicit symbolic,
            scatter/classes={
                acm ={palette4, mark=*,               mark size=2pt},
                cb  ={palette2, mark=+,               mark size=2.5pt},
                cc  ={palette3, mark=x,               mark size=2.5pt},
                cn  ={palette1, mark=star,            mark size=2.5pt},
                eav ={palette9, mark=asterisk,        mark size=2.5pt},
                gos ={palette5, mark=10-pointed star, mark size=2.5pt},
                gra ={palette6, mark=square*,         mark size=2pt},
                sc  ={palette7, mark=triangle*,       mark size=2pt},
                tig ={palette8, mark=diamond*,        mark size=2pt}
            }
        ] table [
            x    index={21},
            y    index={15},
            meta index={0} ] {\allResults};
        \draw[dashed,thick,darkgray] (0.001,0.001) -- (2000,2000);
        \draw[densely dotted,red] (1200,0.001) -- (1200,2000);
        \node[red, rotate=90] at (1700,12) {\tiny \begin{tabular}{c} \textsc{rooted-ibds} timeout  (170 instances) \end{tabular}};
        \draw[densely dotted,red] (0.001,1200) -- (2000,1200);
        \node[red, rotate=0 ] at (22,1500) {\tiny \begin{tabular}{c} \textsc{mixed-ibds} timeout (44 instances) \end{tabular}};
    \end{axis}
\end{tikzpicture}}}
    \caption{
        Scatter plots of the running times, in \emph{seconds}, of the algorithms using a logarithmic scale.
        From left to right, the plots compare the running times of \textsc{mixed-ibds} ($y$-axis) with those of \textsc{exact-ibds}, \textsc{approx-ibds} and \textsc{rooted-ibds}, respectively ($x$-axis).
        Color and shape of instances denote their domain.
        The diagonals (grey dashed lines) separate instances based on which algorithm found a solution faster.
    }
    \label{fig:results}
\end{figure*}

\section{Experimental Evaluation}
    We have compared \textsc{mixed-ibds} against three baseline configurations of \textsc{ibds}:
    \begin{inparaenum}
        \item \textsc{approx-ibds};
        \item \textsc{exact-ibds} (a version that always runs in exact mode by choosing a sufficiently large initial search bound); and
        \item \textsc{rooted-ibds} (a variant of \textsc{mixed-ibds} where canonical $b$-contractions are replaced by rooted $b$-contractions, cf.\ Definition~\ref{def:rooted-b-contr}).\footnote{All configurations rely on the same data structures and functions, except for \textsc{rooted-ibds}, which uses rooted $b$-contractions.}
    \end{inparaenum}
    The algorithms were implemented in C++17 in the novel epistemic planner \textsc{daedalus} (\emph{DynAmic Epistemic and DoxAstic Logic Universal Solver}).\footnote{Available at \url{https://github.com/a-burigana/daedalus}.}
    We tested each configuration on 406 instances (20 minutes timeout) from 9 epistemic planning domains (described in the Appendix): \emph{Active Muddy Child}, \emph{Collaboration through Communication}, \emph{Selective Communication} \cite{conf/aips/Kominis2015}, \emph{Coin in the Box} \cite{journals/corr/Baral2015}, \emph{Consecutive Numbers} \cite{book/sip/vanDitmarschK2015}, \emph{Gossip} \cite{attamahMHD/2014/Knowledge,vanditmarsch/2016/Epistemic,vanditmarschHJPR/2017/Epistemic}, \emph{Grapevine} \cite{conf/aaai/Muise2015} and \emph{Tiger} \cite{herzigAJD/2000/Logic}, plus a novel domain called \emph{Eavesdropping}.
    Recently, \citeauthor{conf/kr/BolanderDH21} (\citeyear{conf/kr/BolanderDH21}) developed a DEL solver that constructs \emph{policies} (mappings from states to actions), which are more general than sequential plans considered in this paper, so we can not directly compare the two approaches.
    
    We now analyze the results of \textsc{mixed-ibds} against the baseline configurations (Figure~\ref{fig:results}) and \textsc{efp 2.0} (Figure~\ref{fig:results-efp}).\footnote{Raw data of tests results are in the Appendix.}

    \noindent\textbf{\textsc{mixed} vs.\ \textsc{exact}.}
    \textsc{exact-ibds} solved 302 instances (74.4\%), while \textsc{mixed-ibds} solved 362 (89.2\%), ${\sim}15\%$ more.
    Since \textsc{exact-ibds} always runs in exact mode, \ie no state is approximated by bounded contractions, bounded contractions behave like standard bisimulation contractions.
    Thus, \textsc{mixed-ibds} computes smaller states than \textsc{exact-ibds}, improving runtimes on the vast majority of the cases, (except for Consecutive Numbers, discussed below), and often finding solutions to instances where \textsc{exact-ibds} timed out, as evidenced by the vertical streak of marks in the time plot (Figure~\ref{fig:results-exact}).
    By computing the \emph{average speedup} (the average of the ratios of the running times) of \textsc{mixed-ibds} wrt.\ \textsc{exact-ibds}, we see that the former configuration is $87.7$ times faster than the latter, with most significant improvements on Active Muddy Child ($211.9{\times}$) and Eavesdropping ($376.7{\times}$), where visited states have a considerable sizes.
    Specifically, in Eavesdropping the size of states grows exponentially after each action, but since \textsc{mixed-ibds} takes $0$-contractions, it collapses states to singleton states, while \textsc{exact-ibds} keeps them in full.
    Finally, in the Consecutive Numbers domain (red star marks), \textsc{mixed-ibds} needs $|\pi|$ iterations to compute plan $\pi$ (cf.\ Example~\ref{ex:ibds}), while \textsc{exact-ibds} only takes one, as the initial search bound is already large enough, and it returns a plan more quickly.

    \noindent\textbf{\textsc{mixed} vs.\ \textsc{approx}.}
    \textsc{approx-ibds} solved 354 instances (${\sim}2\%$ less than \textsc{mixed-ibds}).
    \textsc{mixed-ibds} outperforms \textsc{approx-ibds} in Tiger ($3.1{\times}$), Coin in the Box ($3.4{\times}$), and Grapevine ($74.4{\times}$), and is comparable to \textsc{approx-ibds} in the remaining domains ($6{\times}$ average speedup).

    \noindent\textbf{\textsc{mixed} vs.\ \textsc{rooted}.}
    \textsc{rooted-ibds} solved 236 instances (${\sim}31\%$ less than \textsc{mixed-ibds}).
    The average speedup of \textsc{mixed-ibds} wrt.\ \textsc{rooted-ibds} is $57.2{\times}$.
    As expected, canonical $b$-contractions allow to efficiently check for visited states via set-membership checks.
    Instead, with rooted $b$-contractions, we have to rely on less efficient techniques involving repeated invocations of bounded partition refinement, incurring a significant overhead (cf.\ Section \ref{sec:algorithms}).

    \noindent\textbf{\textsc{mixed-ibds} vs.\ \textsc{efp 2.0}.}
        \textsc{efp 2.0} is based on the $m\mathcal{A}^*$ action language~\cite{journals/corr/Baral2015} and uses a \emph{possibility semantics} for DEL, which improves over Kripke-based implementations~\cite{conf/icaps/Fabiano2020}.
        Since $m\mathcal{A}^*$ does not have the full expressivity of DEL, we could only compare \textsc{efp 2.0} with \textsc{mixed-ibds} on a subset of our benchmark tasks (220 tasks in total, excluding Active Muddy Child, Consecutive Numbers, and Eavesdropping).
        \textsc{efp 2.0} solved 95 instances (42.2\%), while \textsc{mixed-ibds} solved 189 (85.9\%). 
        To account for \textsc{efp 2.0}'s overhead in constructing initial states from formulas, we allowed 10 extra minutes per test, but many still timed out.
        On instances solved by both, \textsc{mixed-ibds} was $279.5{\times}$ faster on average.
        \textsc{efp 2.0} is faster in the Gossip domain (light blue stars above the diagonal in Figure~\ref{fig:results-efp}), as its possibility semantics enables efficient reuse of previously computed information, especially for private announcements~\cite{conf/icaps/Fabiano2020,conf/jelia/BuriganaFM23}, present in Gossip.

    \ifthenelse{\boolean{expandedVersion}}{
        \noindent\textbf{Knowledge Base Approaches.}
        Recently, \textsc{efp 2.0} has been benchmarked against the epistemic planner underlying the \emph{restricted perspectival multi-agent epistemic planning} (RP-MEP) framework \cite{muiseBFMMPS/2022/Efficient}. RP-MEP represents an epistemic state as a set of formulas, called a \emph{proper epistemic knowledge base} (PEKB), and uses belief update and revision to progress PEKBs.
        Moreover, it assumes that the modal depth of formulas in each PEKB is bounded by some constant $d$. This differs fundamentally from our approach, where the bound on reasoning depth is incremented during search. Their planner compiles RP-MEP problems into classical/FOND planning.
        The RP-MEP planner achieves better performance than \textsc{efp 2.0} at shallow depths, but its performance deteriorates when increasing the modal depth of formulas, as its inner compilation step produces plans of exponential size w.r.t.~the original ones.
        Given that \textsc{mixed-ibds} substantially outperforms \textsc{efp 2.0}, it is reasonable to expect that \textsc{mixed-ibds} would be competitive with the RP-MEP planner even at low modal depths, while also scaling more effectively as modal depth increases. Nevertheless, we consider an important future work to provide an experimental comparison of our algorithm and other solvers based on non-DEL approaches, including RP-MEP.
    }{}

    \ifthenelse{\boolean{expandedVersion}}{
    \noindent\textbf{Required depth of reasoning.}
        To better understand which problems can be solved at low depths of reasoning, we computed the average modal depth of the solutions to our benchmarks.
        This value is $6.9$ and drops to $2.4$ when omitting the Consecutive Number domain, specifically considered to force the planner to use maximal reasoning depth. This indicates that many state-of-the-art benchmarks can be solved with a quite limited reasoning bound, which naturally matches our approach.
    }{}

    \begin{figure}[t]
        \centering
        \scalebox{0.8075}{\begin{tikzpicture}
    \centering
    \begin{axis}[
        ymode=log,xmode=log,
        grid=major,
        legend entries={\tiny\textsc{CoinBox},\tiny\textsc{CollabCom},\tiny\textsc{Gossip},\tiny\textsc{Grapevine},\tiny\textsc{SelectCom},\tiny\textsc{Tiger}},
        legend style={at={(0.31,0.97), anchor=north east}},
        legend cell align={left},
        axis x line*=bottom,
        axis y line*=left,
        xmin=0.1,
        xmax=2000,
        ymin=0.1,
        ymax=2000,
        xticklabels={,, $1^{\color{white} 1}$, $10^{\color{white} 1}$, $10^2$, $10^3$},
        yticklabels={,, $1^{\color{white} 1}$, $10^{\color{white} 1}$, $10^2$, $10^3$},
        tick label style={font=\tiny}
        ]
        \addplot[
            scatter,
            only marks,
            point meta=explicit symbolic,
            scatter/classes={
                cb  ={palette2, mark=+,               mark size=2.5pt},
                cc  ={palette3, mark=x,               mark size=2.5pt},
                gos ={palette5, mark=10-pointed star, mark size=2.5pt},
                gra ={palette6, mark=square*,         mark size=2pt},
                sc  ={palette7, mark=triangle*,       mark size=2pt},
                tig ={palette8, mark=diamond*,        mark size=2pt}
            }
        ] table [
            x    index={12},
            y    index={11},
            meta index={0} ] {\efpResults};
        \draw[dashed,thick,darkgray] (0.001,0.001) -- (2000,2000);
        \draw[densely dotted,red] (1200,0.001) -- (1200,2000);
        \node[red, rotate=90] at (1700,12) {\tiny \begin{tabular}{c} \textsc{efp} timeout (125 instances) \end{tabular}};
        \draw[densely dotted,red] (0.001,1200) -- (2000,1200);
        \node[red, rotate=0 ] at (22,1500) {\tiny \begin{tabular}{c} \textsc{mixed-ibds} timeout (31 instances) \end{tabular}};
    \end{axis}
\end{tikzpicture}}
        \caption{Scatter plot of \textsc{mixed-ibds} ($y$-axis) vs.\ \textsc{efp 2.0} ($x$-axis).}
        \label{fig:results-efp}
    \end{figure}

    \section{Discussion}
    Our \textsc{ibds} algorithm provides a novel perspective on DEL-planning.
Instead of searching for plans by iterating on plan length, we iterate on reasoning bound (modal depth).  
The average modal depth of solutions is $6.9$ overall, and $2.4$ when excluding the Consecutive Numbers domain that was specifically designed to exploit maximal reasoning depths. 
This suggests that most planning tasks can be solved with low reasoning depths, an interesting avenue for further exploration.
    Our novel search strategy improves state-of-the-art (\textsc{efp 2.0}), showing that 
    fully general DEL planners can be efficient. We conjecture that 
     our approach can also be adapted to outperform the planner of \cite{conf/kr/Bolander2021},
     as their method 
     relies on standard bisimulation contractions only.

    Our paper considered only \emph{single-pointed} states, but the \emph{multi-pointed} case can be covered by translating into single-pointed states~\cite{conf/kr/BolanderDH21}.
    
    As shown, \textsc{mixed-ibds} performs better than \textsc{exact-ibds} in Consecutive Numbers, as it requires more iterations to find a bound that admits a plan.
    We plan to tackle this issue in the future, e.g.\ by doubling the bound at each iteration, and/or via heuristics to ``guess'' a better initial bound.

    \myparagraph{Acknowledgements}
    This work is partially funded by the NextGenerationEU FAIR PE0000013 project MAIPM (CUP C63C22000770006).  Thomas Bolander is supported by the Independent Research Fund Denmark (grant
no. 10.46540/4258-00060B). 

    \bibliographystyle{kr}
    \bibliography{bibliography}
\end{document}